\documentclass[letterpaper]{article} 
\usepackage{aaai23}  
\usepackage{times}  
\usepackage{helvet}  
\usepackage{courier}  
\usepackage[hyphens]{url}  
\usepackage{graphicx} 
\graphicspath{ {./figs/} }
\usepackage{caption} 
\usepackage{algorithm,algpseudocode, algorithmicx}

\usepackage{newfloat}
\usepackage{listings}
\DeclareCaptionStyle{ruled}{labelfont=normalfont,labelsep=colon,strut=off} 
\floatstyle{ruled}
\newfloat{listing}{tb}{lst}{}
\floatname{listing}{Listing}
\usepackage{hyperref}

\title{Tackling Data Heterogeneity in Federated Learning with Class Prototypes}

\author{
    Yutong Dai,\textsuperscript{\rm 1}
    Zeyuan Chen,\textsuperscript{\rm 2}
    Junnan Li,\textsuperscript{\rm 2}
    Shelby Heinecke,\textsuperscript{\rm 2}
    Lichao Sun,\textsuperscript{\rm 1}
    Ran Xu\textsuperscript{\rm 2}
}
\affiliations {
    \textsuperscript{\rm 1} Lehigh University,
    \textsuperscript{\rm 2} Salesforce Research\\
    \{yud319, lis221\}@lehigh.edu, \{zeyuan.chen, junnan.li, shelby.heinecke, ran.xu\}@salesforce.com
}

\usepackage{mathtools,amssymb,amsfonts,amsthm,bm} 
\usepackage{lipsum} 
\usepackage{booktabs}
\usepackage{multirow}
\usepackage{caption}
\usepackage{dashrule}
\usepackage{xcolor}
\usepackage[labelformat=simple]{subcaption}

\newtheorem{theorem}{Theorem}
\newtheorem{lemma}{Lemma}

\theoremstyle{definition}

\newtheorem{assumption}{Assumption}

\newcommand{\fedavg}{{\text{FedAvg}}}
\newcommand{\fedavguh}{{\text{FedAvg+UH}}}

\newcommand{\grad}{\nabla}
\newcommand{\norm}[1]{\left\|{#1}\right\|}

\newcommand{\R}[1]{\mathbb{R}^{#1}}
\newcommand{\Embb}{\mathbb{E}}
\newcommand{\Ncal}{{\cal N}}
\newcommand{\Ccal}{{\cal C}}
\newcommand{\Scal}{{\cal S}}

\newcommand{\Dcal}{{\cal D}}
\newcommand{\Ocal}{{\cal O}}
\newcommand{\Bcal}{{\cal B}}
\newcommand{\Pmbb}{\mathbb{P}}
\begin{document}

\maketitle

\begin{abstract}
 
\end{abstract}

\section{Introduction}\label{sec:introduction}

Federated learning (FL)~\cite{mcmahan2017communication} is an emerging area that attracts significant interest in the machine learning community due to its capability to allow collaborative learning from decentralized data with privacy protection. However, in FL, clients may have different data distributions, which violates the standard independent and identically distribution (i.i.d) assumption in centralized machine learning. The non-i.i.d phenomenon is known as the data heterogeneity issue and is an acknowledged cause of the performance degradation of the global model~\cite{tan2022towards}. Moreover, from the client's perspective, the global model may not be the best for their tasks. Therefore, personalized federated learning (PFL) emerged as a variant of FL, where personalized models are learned from a combination of the global model and local data to best suit client tasks.

While PFL methods address data heterogeneity, class imbalance combined with data heterogeneity remains overlooked. Class imbalance occurs when clients' data consists of different class distributions and the client may not possess samples of a particular class at all. Ideally, the personalized model can perform equally well in all classes that appeared in the local training dataset. For example,   medical institutions have different distributions of medical records across diseases \cite{ng2021federated}, and it is crucial that the personalized model can detect local diseases with equal precision. Meanwhile, the currently adopted practice of evaluating the effectiveness of PFL methods can also be biased. Specifically, when evaluating the accuracy, a single balanced testing dataset is split into multiple local testing datasets that match clients' training data distributions. Then each personalized model is tested on the local testing dataset, and the averaged accuracy is reported. However, in the presence of class imbalance, such an evaluation protocol will likely give a biased assessment due to the potential overfitting of the dominant classes. It is tempting to borrow techniques developed for centralized class imbalance learning, like re-sampling or re-weighting the minority classes. However, due to the data heterogeneity in the FL setting, different clients might have different dominant classes and even have different missing classes; hence the direct adoption may not be applicable. Furthermore, re-sampling would require the knowledge of all classes, potentially violating the privacy constraints.

Recent works in class imbalanced learning in non-FL settings \cite{kang2019decoupling, zhou2020bbn} suggest decoupling the training procedure into the representation learning and classification phases. The representation learning phase aims to build high-quality representations for classification, while the classification phase seeks to balance the decision boundaries among dominant classes and minority classes. Interestingly, FL works such as \cite{oh2021fedbabu,chen2021bridging} find that the classifier is the cause of performance drop and suggest that learning strong shared representations can boost performance. 

Consistent with the findings in prior works, as later shown in Figure \ref{fig:toy_example}, we observe that representations for different classes are uniformly distributed over the representation space and cluster around the class prototype when learned with class-balanced datasets. However, when the training set is class-imbalanced, as is the case for different clients, representations of minority classes overlap with those of majority classes; hence, the representations are of low quality. Motivated by 
these observations, we propose FedNH (\underline{n}on-parametric \underline{h}ead), a novel method that imposes uniformity of the representation space and preserves class semantics to address data heterogeneity with imbalanced classes. We initially distribute class prototypes uniformly in the latent space as an inductive bias to improve the quality of learned representations and smoothly infuse the class semantics into class prototypes to improve the performance of classifiers on local tasks. Our contributions are summarized as follows.

\begin{itemize}
    \item We propose FedNH, a novel method that tackles data heterogeneity with class imbalance by utilizing uniformity and semantics of class prototypes. 
    \item We design a new metric to evaluate personalized model performance. This metric is less sensitive to class imbalance and reflects personalized model generalization ability on minority classes.
    \item Numerical experiments on Cifar10, Cifar100, and TinyImageNet show that FedNH can effectively improve both personalized and global model classification accuracy. The results are on par or better than the state-of-the-art, with significantly fewer computation costs (refer to Appendix~\ref{appendix:experiments} for discussions on computation costs).
\end{itemize}

We close this section by introducing the notation and terminology used throughout the paper.

\noindent\textbf{Notation and Terminology.}
Let $\R{n}$ and $\R{m\times n}$ denote the set of $n$ dimensional real vectors and the set of $m$-by-$n$ dimensional real matrices, respectively. Unless stated explicitly, $\norm{\cdot}$ denotes the $\ell_2$ norm and $|\cdot|$ denotes the set cardinality operator. Let $[N]$ denote the set $\{1,2,\cdots, N\}$ for any positive integer $N$. For any matrix $A\in\R{m\times n}$, $A_i$ denotes the $i$th row of $A$. 
Let $\Ncal_d(\mu,\sigma)$ denote the d-dimensional normal distribution with mean $\mu$ and variance $\sigma^2$. $X\sim \Ncal_d(\mu,\sigma)$ represents a random sample from the distribution. $\Embb[\cdot]$ is the expectation operator and $\lceil\cdot\rceil$ is the round-up operator.

For a neural network, we decompose its parameters into the body ($\theta$) and the head ($W$).
The body is used to learn the abstract representation of inputs, and the head is used for classification. The output of a neural network can be written as $Wf(\theta;\cdot)$, where the second argument is a placeholder for an input. In this work, we solely consider the last linear layer as the head. We use the terms \textit{head} and \textit{prototype} interchangeably as they both refer to parameters of the last linear classification layer of the neural network. The term head is often used in discussing the neural network's architecture, while the term prototype appears more often when discussing the classification task. Some works define prototype as the averaged representations of a class, and such a distinction will be made clear in the context.

\begin{figure*}[!t]
     \centering
     \resizebox{1.0\textwidth}{!}{
     \begin{subfigure}[b]{0.24\textwidth}
         \centering
         \includegraphics[height=3.5cm]{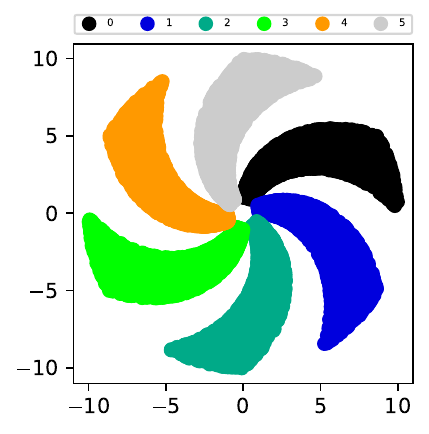}
         \caption{Spiral Dataset}
         \label{fig:balanced-trainset}
     \end{subfigure}
    \hfill     
     \begin{subfigure}[b]{0.24\textwidth}
         \centering
         \includegraphics[height=3.5cm]{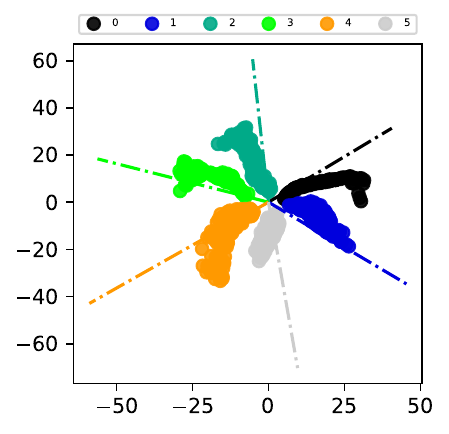}
         \caption{Balanced + FH}
         \label{fig:b_fh}
     \end{subfigure}
    \hfill
    \begin{subfigure}[b]{0.24\textwidth}
         \centering
         \includegraphics[height=3.5cm]{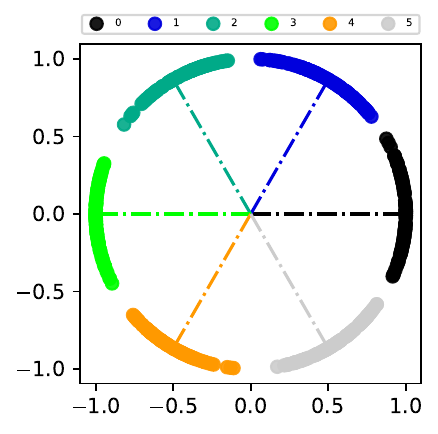}
         \caption{Balanced + UH}
         \label{fig:b_uh}
    \end{subfigure}  
     \hfill
     \begin{subfigure}[b]{0.24\textwidth}
         \centering
         \includegraphics[height=3.5cm]{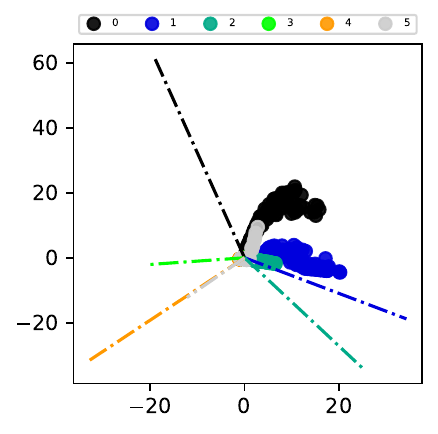}
         \caption{Imbalanced + FH}
         \label{fig:ub_fh}
     \end{subfigure}
    \hfill
    \begin{subfigure}[b]{0.24\textwidth}
         \centering
         \includegraphics[height=3.5cm]{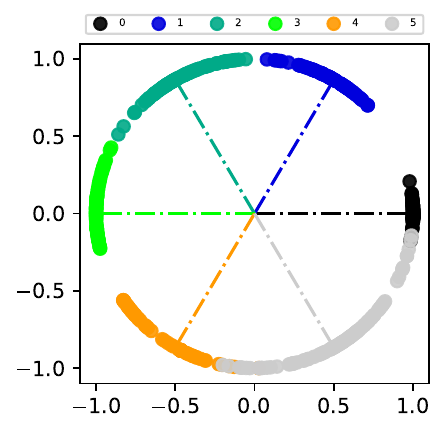}
         \caption{Imbalanced + UH}
         \label{fig:ub_uh}
    \end{subfigure}
    }
    \caption{Visualization of the spiral dataset and the learned representations under different methods. (a): Balanced training dataset. (b): Representations (dots) and class prototypes (dashed lines) are learnable under the balanced dataset with the free head (FH), i.e, the classification head is also learnable. (c): Representations and class prototypes learned under the balanced dataset with the uniform head (UH). (d) and (e) are parallel to (c) and (e), respectively, where the training dataset is imbalanced.}
    \label{fig:toy_example}
\end{figure*}

\section{Related Work}
\subsection{Personalized Federated Learning}
PFL methods can be roughly classified based on the strategies to generate personalized models, e.g., parameter decoupling, regularization, and model interpolation. For a detailed discussion, we refer readers to the survey~\cite{tan2022towards}. Here, we mainly focus on methods that decouple the body and head when clients perform local training.
FedPer~\cite{arivazhagan2019federated} learns the body and head concurrently as in FedAvg and only
shares the body with the server. Therefore, a personalized model consists of a shared body and a personalized head. FedRep~\cite{collins2021exploiting} learns the head and body sequentially and only shares the body. Specifically, each client first learns the head with a fixed body received from the server and then learns the body with the latest personalized head fixed.
FedBABU~\cite{oh2021fedbabu} learns only the body with the randomly initialized and \textit{fixed} head during local updates and shares only the body to the server. Personalized models are obtained by fine-tuning the global model when training is finished. FedROD~\cite{chen2021bridging} designs a two-head-one-body architecture, where the two heads consist of a generalized head trained with class-balanced loss while the personalized head is trained with empirical loss. The body and generalized head are shared with the server for aggregation, and the personalized head is kept privately. The methods mentioned above assume that the model architectures are the same across clients.
FedProto~\cite{tan2022fedproto} lifts such a restriction. It only shares the class prototypes (calculated as the averaged representations of each class) so that different clients can have different model architectures as bodies.
\subsection{Class Imbalanced Learning}
In the non-FL setting,  data-level and algorithm-level approaches are the most common ways to address the class imbalance. Over-sampling minority classes or under-sampling majority classes are simple yet effective ways to create more balanced class distributions~\cite{kubat1997addressing, chawla2002smote,  he2009learning}. However, they may be vulnerable to overfitting minority classes or information loss on majority classes. On the algorithm level, various class-balanced losses are proposed to assign different losses either sample-wise or class-wise~\cite{lin2017focal,khan2017cost,cao2019learning,cui2019class}. Recent works \cite{kang2019decoupling, zhou2020bbn} suggest decoupling the training procedure into the representation learning and classification phases, where strong representations are learned during the first phase while the classifier is re-balanced during the second phase. \cite{wang2020long} proposes a multi-expert framework, where each expert is only responsible for a subset of classes to minimize the bias on minority classes.

In the FL setting, a few works address class imbalance. \cite{duan2020self}~selects clients with complimentary class distributions to perform updates, which requires clients to reveal the class distribution to server. \cite{wang2021addressing} assumes an auxiliary dataset with balanced classes are available on the server to infer the composition of training data and designs a ratio loss to mitigate the impact of imbalance. CReFF~\cite{shang2022federated} extends the idea from~\cite{kang2019decoupling} to re-train a classifier based on federated features in a privacy preserving manner. To our best knowledge, FedROD and CReFF are the only works that address the data heterogeneity with class imbalance in the FL setting without potential privacy leakage and without requiring auxiliary datasets on the server side.

\section{Methodology}

\subsection{A Motivating Example}\label{sec:motivating.example}
In PFL, a crucial part of training is done locally in each client. We motivate our method by considering the behavior of learned representations of inputs and class prototypes learned by a neural network with cross-entropy loss under both balanced and imbalanced training datasets. We first generate a balanced two-dimensional synthetic spiral training dataset (Figure~\ref{fig:balanced-trainset}) with six classes, each with 3000 points. For each class $k\in\{0,\cdots,5\}$, the data points are generated as 
$\Ccal_k=\{(x_{k,i}, y_{k,i})~|~ x_{k,i}=r_i\sin \omega_{k,i}, y_{k,i}=r_i\cos \omega_{k,i}, i\in[3000]\}$, where for all $i\in[3000]$,
$$
 r_i = 1+(i-1)\frac{9}{2999} \text{ and } \omega_i=\frac{k}{3}\pi + (i-1)\frac{k}{3\times 2999}\pi + b_i,
$$
and $b_i\sim \Ncal_1(0,1)$.
A balanced testing dataset is created similarly (Figure~\ref{fig:appendix-spiral-testing} in Appendix~\ref{appendix:experiments}). We consider a 5-layer fully connected neural network with ReLu activation~\cite{agarap2018deep} as the classifier. For visualization purposes, the dimension of the latent representation is set to two. We draw latent representations (points) and class prototypes (dashed lines)\footnote{We recall that the weight matrix of the last linear layer represents the prototypes, and each row of the matrix is also a vector of the same dimension as the latent representation.} over the testing dataset in Figure~\ref{fig:b_fh}. Two key observations include that 1) the norm of class prototypes are of about the same magnitude and 2) class prototypes distribute relatively uniformly over the latent space. Next, we generate an imbalanced training dataset, where the number of samples for each class is selected sequentially from $\{3000, 1500, 750, 375, 187, 93\}$. Then as shown in Figure~\ref{fig:ub_fh}, the learned class prototypes are not of the same magnitude, the minority classes' prototypes are pulled towards the majority classes' prototypes, and the representations of minority classes collapse into the space that the majority ones occupy. This phenomenon matches with the previous study \cite{kang2019decoupling} and motivates us to distribute class prototypes uniformly over a unit hypersphere and fix prototypes throughout the training phase. As shown in Figure~\ref{fig:b_uh}, enforcing class prototypes' uniformity will not hurt the training when the training dataset is balanced. And when the training set is imbalanced, from Figure~\ref{fig:ub_uh}, one can see the uniformity of class prototypes is an effective inductive bias to combat the negative impact of the dominating classes over the minority classes. We end this section by making the following remarks.
\begin{itemize}
    \item Enforcing uniformity over the class prototypes can be regarded as a method of reweighing the training samples.
    For the samples that cannot be easily and correctly classified, they make more contributions to the cross-entropy loss. This uniformity strategy differs from methods such as Focal loss~\cite{lin2017focal}, that add explicit weights per sample. Meanwhile, this strategy does not require tuning any parameters. Moreover, it brings additional benefits to the minority classes, whose representations are no longer overlapped with the majority classes.

    \item Fixing class prototypes sets a consistent learning goal for clients~\cite{oh2021fedbabu} and imposing uniformity over class prototypes helps to combat minority representation collapse in the FL setting. We empirically validate this claim by distributing the spiral training set to 100 clients and perform learning with FedAvg. More details for the experiment can be found in Appendix~\ref{appendix:experiments}. Figure~\ref{fig:fedavgc0_spiral} and Figure~\ref{fig:fedavgc1_spiral} visualize class prototypes and latent feature representations for two different clients learned with FedAvg. Since prototypes are free to move, both prototypes and representations for the same class occupy different places in different clients. Fixing prototypes with a uniformity constraint addresses such an issue as shown in Figure~\ref{fig:feduhc0_spiral} and Figure~\ref{fig:feduhc1_spiral}.
    
    \item Independently, a related idea is also explored in non-FL setting~\cite{li2022targeted}. However, our method is different from this work since we our class prototypes takes into account the class semantic similarity while class prototypes are fixed and never get updated in~\cite{li2022targeted}.
\end{itemize}

\begin{figure}[!ht]
     \centering
     \begin{subfigure}[b]{0.23\textwidth}
         \centering
         \includegraphics[width=3.2cm]{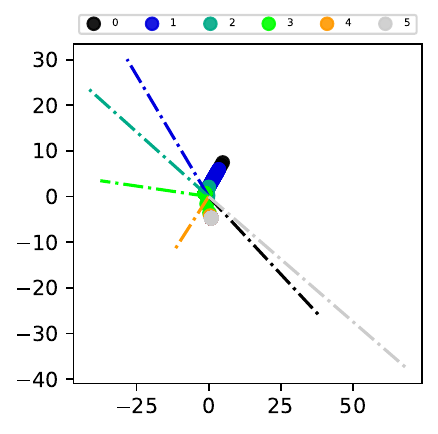}
         \caption{Client 0 of \fedavg{} }
         \label{fig:fedavgc0_spiral}
     \end{subfigure}
     \hfill
     \begin{subfigure}[b]{0.23\textwidth}
         \centering
         \includegraphics[width=3.2cm]{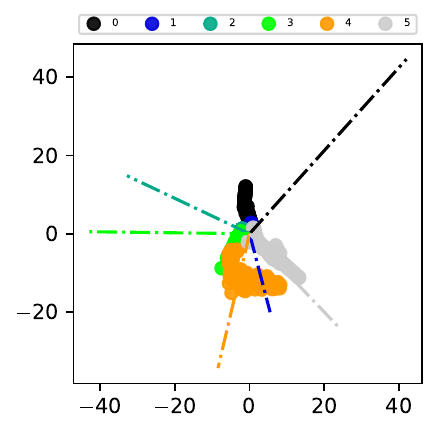}
         \caption{ Client 1 of \fedavg{}}
         \label{fig:fedavgc1_spiral}
     \end{subfigure}
    \hfill
    \begin{subfigure}[b]{0.23\textwidth}
         \centering
         \includegraphics[height=3.2cm]{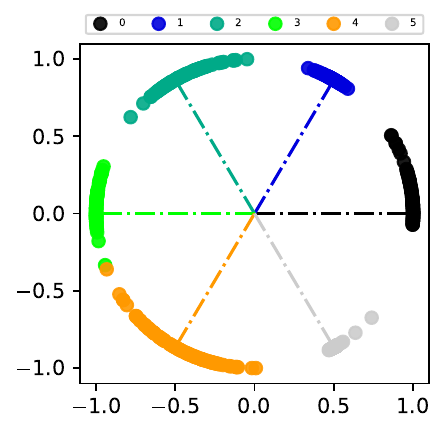}
         \caption{Client 0 of \fedavguh{} }
         \label{fig:feduhc0_spiral}
    \end{subfigure}  
    \hfill
     \begin{subfigure}[b]{0.23\textwidth}
         \centering
         \includegraphics[height=3.2cm]{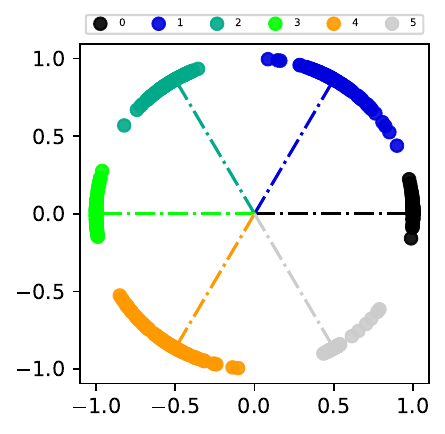}
         \caption{Client 1 of \fedavguh{} }
         \label{fig:feduhc1_spiral}
     \end{subfigure}     
    \caption{Learned representations are consistent for different clients for FedAvg with the uniform classification head compared with vanilla FedAvg. }
    \label{fig:consistent.goal}
\end{figure}

\subsection{Proposed Method}
As we demonstrated in the motivating example, class prototypes learned from balanced datasets are distributed uniformly in the latent space. Hence, encouraging class prototypes learned from unbalanced data to also be distributed uniformly may be a way to improve performance. We also note that since the motivating example was produced with synthetic data, class semantics were not discussed. In practice, some classes are more semantically similar to others. For example, the (wolf, dog) pair is more similar than the (wolf, plant) pair.
Capturing semantics helps with performance~\cite{su2012improving}.
This leads to two questions: 
\begin{itemize}
    \item (1) Given imbalanced data, how can we generate class prototypes such that pairs of prototypes are maximally separated and equidistant? 
    \item (2) How can we encourage class prototypes to learn class semantics?
\end{itemize}
We address the first question by proposing a parameter initialization scheme for the classification head during FL training. We answer the second question by proposing an approach to infuse class semantics throughout the training process. Both answers are described in more detail in the following sections.

\subsubsection{Initialization.} The server is responsible for initializing parameters for both the body and classification head. The parameters for the body are initialized using common techniques, like Kaiming initialization~\cite{he2015delving} and Xavier initialization~\cite{glorot2010understanding}. However, we initialize the parameter $W$ for the head in a way to address the first question. Specifically, assuming there are $|\Ccal|$ classes and the latent representation lives in a $d$-dimensional space, then
one can solve a constrained optimization problem~\eqref{prob:max-equi-sep}.
\begin{align}
& \max_{\{w_1,\cdots,w_{|\Ccal|},M\}}  M^2 \label{prob:max-equi-sep}\\
& \text{s.t.}  \norm{w_i-w_j}^2 \geq M^2, \norm{w_i}^2 = 1 \text{ for all } i\in[|\Ccal|], i\neq j\nonumber.
\end{align}
The first constraint is imposed to satisfy the maximally separated and equidistant requirements, while the second constraint is added to ensure all class prototypes are of the same magnitude. The problem can be solved by, for example, the interior point method \cite{wright1999numerical}. We remark that \cite{li2022targeted} uses a different approach to generate class prototypes, which does not ensure that each prototype is of the same magnitude. When $d$ and $|\Ccal|$ are large, the computation cost might be high. However, the computation only needs to be performed once at the initialization stage and saved for future uses. Lastly, we remark that the body of the network internally adds a normalization layer to its output such that $\norm{f(\theta;x)}=1$ for any $\theta$ and valid input $x$. This normalization technique adds no extra parameters and negligible computation cost. It is also used in face verification tasks \cite{wang2017normface,wang2018additive}.

\subsubsection{Client Update.} For the $t$th communication round, the $k$th client receives the body parameter $\theta^t$ and the head parameter $W^t$. Together with its local training dataset $\Dcal_k$, it aims to achieve two goals. The first goal is to learn a strong body that can extract representations used for classification. To this goal, the $k$th client approximately solves the problem~\eqref{prob:local-prob}.

\begin{align}\label{prob:local-prob}
    \theta_{k}^{t+1}\approx 
    \arg\min_{\theta} F(\theta;\Dcal_k,W^t,s) \text{ with } F(\theta;\Dcal_k,W^t,s):= \frac{1}{|\Dcal_k|} \sum_{(x_i,y_i)\in\Dcal_k}-\log\frac{\exp(s\cdot W_{y_i}^{t}f(\theta;x_i))}{\sum_{j=1}^{|\Ccal|}\exp(s\cdot W_j^{t}f(\theta;x_i))},
\end{align}
where $s$ is a fixed constant scalar. The $s$ is added to compensate for the fact that $\norm{W_{y_i}^{t}f(\theta;x_i))}\leq 1$ \footnote{Both the latent representation $f(\theta;x_i)$ and the prototype $W_{y_i}^{t}$ are normalized.}.  Specifically, the loss $F(\theta;\Dcal_k,W^t,s)$ is lower bounded by a quantity inversely proportional to $s\cdot\norm{W_{y_i}^{t}f(\theta;x_i))}$ \cite[Proposition 2]{wang2017normface}. If $s=1$, it hinders the optimization process of learning a strong body. In practice, $s$ can be set to some constants like 30 as suggested in \cite{wang2018additive}. The inverse of $s$ is also called temperature in knowledge distillation \cite{hinton2015distilling}. 

To approximately solve problem~\eqref{prob:local-prob}, we assume $J$ steps of stochastic gradient descent are performed, i.e.,
\begin{equation*}\label{eq:sgd-update}
    \theta_{k}^{t,j}=\theta_{k}^{t,j-1} - \eta_t G_k(\theta_{k}^{t,j-1};W^t) \text{ for } j\in[J],
\end{equation*}
where $\theta_k^{t,j}$ is the body parameter at the $k$th client in $t$th communication round after the $j$th local update and $G_k(\theta_{k}^{t,j-1};W^t)$ is the stochastic gradient estimator (See Assumption~\ref{ass:first-second-momentum}). It is clear that $\theta_{k}^{t,J}=\theta_{k}^{t}$ and $\theta_{k}^{t,0}=\theta^{t-1}$. It is worth mentioning that the classification head is fixed during the local training so that all selected clients have a consistent learning goal under the guidance of the same head $W^t$.

Once the local training is done, the client completes its second goal, i.e., computing the averaged representations per class contained in $\Ccal_k$ as \eqref{eq:local-proto-update}
\begin{equation}\label{eq:local-proto-update}
    \mu_{k,c}^{t+1} = 
    \begin{cases}
        \frac{1}{|\{i: y_i=c\}|}\sum_{\{i: y_i=c\}}r_{k,i}^{t+1}, & \text{ if } c\in\Ccal_k.\\
        0, &\text{ o.w. },
    \end{cases}
\end{equation}
where $\Ccal_k$ is the set of classes that the $k$th client owns and $r_{k,i}^{t+1}=f(\theta_{k}^{t+1};x_i,y_i)$ for ${(x_i,y_i)}\in\Dcal_k$. These averaged representations $\{\mu_{k,c}^{t+1}\}_{c\in\Ccal_k}$ provide a succinct local description of class prototypes for classes from $\Ccal_k$. Finally, both the body parameters $\theta_{k}^{t+1}$ and local prototypes $\{\mu_{k,c}^{t+1}\}_{c\in\Ccal_k}$ are sent back to the server for aggregation. 

\subsubsection{Server Update.} For the $t$th communication round, assume that clients from a set $\Scal^t$ respond to the server. Then the server aggregates the body by taking the average of the received $\{\theta_{k}^{t+1}\}_{k\in\Scal^t}$, which is the same as in FedAvg. 
However, for the global prototype update, we propose a new strategy that infuses class semantics, addressing the second question.
Specifically, for $c\in \Ccal$,
\begin{equation}\label{eq:prototype-update}
    W_c^{t+1} \gets \rho W_c^{t} +  (1-\rho) \sum_{k\in S^t}\alpha_{k}^{t+1}\mu^{t+1}_{k,c},
\end{equation}
where the aggregation weights $\{\alpha_k^{t+1}\}_{k\in\Scal^t}$ are some positive scalars and with smoothing parameter $\rho\in(0,1)$. A specific choice on these parameters are discussed in Appendix~\ref{appendix:impl.details}. We remark that $\rho$ is suggested to set close to 1 for two reasons. First, at the initial training stage, the body of a network is not well trained; hence the learned representations and the clients' proposed prototypes are less informative. In this stage, we utilize the uniformity prototypes initialized at the server to guide the local training. Second, due to the client sampling, it is ideal that prototypes at the server side changes gradually, hence using the $1-\rho$ to discount the contribution of local clients to the update of prototypes at the server side. The importance of $\rho$ will also be justified by the convergence analysis in Section~\ref{sec:convergence}. It is worth mentioning class prototypes are updated in a non-parametric manner instead of by gradients. 
We end this section by formally summarizing our method FedNH in Algorithm~\ref{alg:FedNH}.
\begin{algorithm}[t]
\algnewcommand\algorithmicinit{\textbf{Initialization:}}
\algnewcommand\Initialization{\item[\algorithmicinit]}
\algnewcommand\algorithmicoutput{\textbf{Output:}}
\algnewcommand\Output{\item[\algorithmicoutput]}
\caption{FedNH}
\begin{algorithmic}[1]
    \Require Total number of clients $K$; participation rate $\gamma$; number of communication rounds $R$; the smoothing parameter $\rho\in(0,1)$; a set of full classes $\Ccal$; the set of classes that the $k$th client owns as $\Ccal_k$.
    \Initialization  Set the class prototype $W\in\R{|\Ccal|\times d}$ with $W^0$ and the $\theta$ with $\theta^0$.
        \For{$t=0,\dots,R-1$ communication rounds}
            \State Sample a subset $\Scal^t$ of clients with $|\Scal^t| = \lceil\gamma K\rceil$.
            \State Broadcast $\{\theta^t, W^t\}$ to clients $k\in\Scal^t$.
            \For {each client $k \in\Scal^t$ \textbf{in parallel}}
                \State $\left(\theta_{k}^{t+1}, \mu_k^{t+1}\right) \gets \text{ClientUpdate}(\theta^t, W^t)$.
            \EndFor
            \State Perform the global class prototype update as \eqref{eq:prototype-update}.
            \State Normalize $W_c$ to have unit norm for all $c\in \Ccal$.
            \State Perform  aggregation as $\theta^{t+1} = \frac{1}{|\Scal^t|}\sum_{k\in\Scal^t}\theta_k^{t}.$
        \EndFor
    \Procedure{ClientUpdate}{$\theta^t, W^t$}
        \State Initialize the local representation network with $\theta^t$ and class prototypes with $W^t$.
        \State Approximately solve the problem \eqref{prob:local-prob} to obtain
        $\theta_{k}^{t+1}$.
        \State Perform the local class prototypes update as \eqref{eq:local-proto-update}.
        \State \textbf{Return} $\left(\theta_{k}^{t+1}, \mu_{k}^{t+1}\right)$.
\EndProcedure
\end{algorithmic}
\label{alg:FedNH}
\end{algorithm}

\subsection{Convergence Analysis}\label{sec:convergence}
To avoid cluttered notations, we denote the local loss $F_k(\theta;W)\equiv F(\theta;\Dcal_k,W,s)$. We first present a set of assumptions required to perform the convergence analysis.
\begin{assumption}\label{ass:smoothness}
For all $k\in[K]$, $F_k(\theta;W)$ is bounded below and $F_k(\theta;W)$ is $L_g$ smooth with respect to its first argument, i.e., for any fixed $W\in\R{d_2}$ and $\text{ for all } (\theta_1,\theta_2) \in \R{d_1}\times\R{d_1}.$
$$
\norm{\grad_\theta F_k(\theta_1;W)-\grad_\theta F_k(\theta_2;W)}\leq L_g\norm{\theta_1-\theta_2}. 
$$
\end{assumption}

\begin{assumption}\label{ass:first-second-momentum}
Denote $G_k(\theta;W)$ as the stochastic gradient, where only a random subset of $\Dcal_k$ is used to estimate the true partial gradient $\grad_\theta F_k(\theta;W)$. There is a constant $\sigma>0$, such that for all $k\in[K]$ and $(\theta, W)\in\R{d_1}\times \R{d_2}$,
$$
\begin{aligned}
  &\Embb[G_k(\theta;W)] = \grad_\theta F_k(\theta;W) \text{ and } \Embb[\norm{G_k(\theta;W) - \grad_\theta F_k(\theta;W)}^2]\leq \sigma^2.  
\end{aligned}
$$
\end{assumption}

\begin{assumption}\label{ass:bounded-gradient}
Assume $\Embb\left[\norm{G(\theta;W)}\right]^2\leq M_G^2$ and $\norm{\grad_\theta f(\theta;x,y)}\leq M_f$  for all $(\theta,W)\in \Bcal$ and $(x,y)\in\bigcup_k \Dcal_k$, where $\Bcal$ is a bounded set.
\end{assumption}

Assumption~\ref{ass:smoothness} is parallel to the common smoothness assumption made in the literature\cite{li2020federated,karimireddy2020scaffold,acar2020federated}, where we slightly loose it by only requiring $F_k$ to be smooth with respect to its first argument. Assumption~\ref{ass:first-second-momentum} is also a standard assumption on the first and second order momentum of the stochastic gradient estimate $G_k(\theta;W)$. Assumption~\ref{ass:bounded-gradient} imposes bound on the gradient, which is reasonable as long as $(\theta, W)$ remains bounded and it is also assumed in \cite{stich2019local, yu2019parallel, tan2022fedproto}.

We aim to establish the convergence results for each client. The challenge in the analysis is that, for each round $t$, the objective function $F_k(\theta;W^t)$ is different as $W^t$ changes. Thus, the type of convergence result we can expect is that for some $(\theta_k^{t,j},W^t)$, the $\norm{\grad_\theta F_k(\theta^{t,j};W^t)}$ is small. An informal statement of the convergence result is given below. The formal statement and proofs can be found at Theorem~\ref{thm:formal} in Appendix~\ref{appendix:complated-convergence-analysis}.
\begin{theorem}[Informal]\label{thm:convergence}
Let the $k$th client uniformly at random returns an element from $\{\theta_{k}^{t,j}\}$ as the solution, denoted as $\theta_k^*$. Further, let $W^*$ share the same round index as $\theta_k^*$. Then for any $\epsilon>0$, set $\rho\in(\nu_1(\epsilon,M_G, M_f), 1)$ and $\eta\in \left(0, \nu_2(\epsilon, L_g, \sigma^2, \rho, M_G, M_f)\right)$, if $R>\Ocal(\epsilon^{-1})$, one gets
$$
\Embb\left[\norm{\grad_{\theta} F_k(\theta_{k}^{*};W^*)}^2\right] \leq \epsilon,
$$
where $\nu_1(\epsilon,M_G, M_f)$, $\nu_2(\epsilon, L_g, \sigma^2, \rho, M_G, M_f)$, $M_G$, and $M_f$ are some positive constants.
\end{theorem}

\begin{table*}[!th]
\fontsize{9}{9}
\begin{tabular}{@{}c|lcccccc@{}}
\toprule
Dataset                                            & \multicolumn{1}{c|}{Method}    & \multicolumn{3}{c|}{Dir(0.3)}                                                               & \multicolumn{3}{c}{Dir(1.0)}                                          \\ \midrule
\multicolumn{1}{c|}{}                              & \multicolumn{1}{c|}{}         & GM                    & PM(V)                 & \multicolumn{1}{c|}{PM(L)}                 & GM                    & PM(V)                 & PM(L)                 \\ \midrule
\multicolumn{1}{c|}{\multirow{9}{*}{Cifar10}}      & \multicolumn{1}{l|}{Local}    & ---                   & 42.79 ± 2.45          & \multicolumn{1}{c|}{71.57 ± 1.82}          & ---                   & 41.20 ± 1.27          & 58.34 ± 1.03          \\
\multicolumn{1}{c|}{}                              & \multicolumn{1}{l|}{FedAvg}   & 66.40 ± 3.13          & 63.10 ± 1.33          & \multicolumn{1}{c|}{\underline{84.08 ± 2.31}} & 73.07 ± 1.60          & 68.07 ± 1.23          & \underline{79.12 ± 2.11} \\
\multicolumn{1}{c|}{}                              & \multicolumn{1}{l|}{FedPer}   & 61.58 ± 0.43          & 59.66 ± 2.34          & \multicolumn{1}{c|}{82.38 ± 1.50}          & 63.33 ± 0.53          & 60.66 ± 2.17          & 73.40 ± 1.36          \\
\multicolumn{1}{c|}{}                              & \multicolumn{1}{l|}{Ditto}    & 66.40 ± 3.13          & 53.69 ± 1.11          & \multicolumn{1}{c|}{80.08 ± 2.17}          & 73.07 ± 1.60          & 61.22 ± 1.77          & 74.78 ± 2.16          \\
\multicolumn{1}{c|}{}                              & \multicolumn{1}{l|}{FedRep}   & 40.13 ± 0.17          & 56.47 ± 2.31          & \multicolumn{1}{c|}{80.22 ± 2.45}          & 47.92 ± 0.38          & 55.06 ± 2.27          & 68.99 ± 1.27          \\
\multicolumn{1}{c|}{}                              & \multicolumn{1}{l|}{FedProto} & ---                   & 41.48 ± 1.02          & \multicolumn{1}{c|}{68.35 ± 1.75}          & ---                   & 39.65 ± 1.33          & 53.23 ± 1.78          \\
\multicolumn{1}{c|}{}                              & \multicolumn{1}{l|}{CReFF}    & 66.46 ± 1.40          & 63.10 ± 2.16          & \multicolumn{1}{c|}{\underline{84.08 ± 2.31}} & 71.63 ± 0.61       & 68.07 ± 1.44         & \underline{79.12 ± 2.11}          \\
\multicolumn{1}{c|}{}                              & \multicolumn{1}{l|}{FedBABU}  & 62.78 ± 3.09          & 60.58 ± 2.16          & \multicolumn{1}{c|}{82.64 ± 1.03}          & 70.34 ± 1.72          & 65.49 ± 1.44          & 77.35 ± 1.80          \\
\multicolumn{1}{c|}{}                              & \multicolumn{1}{l|}{FedROD}   & \textbf{72.31 ± 0.16} & \textbf{65.66 ± 1.27} & \multicolumn{1}{c|}{83.44 ± 1.03}          & \textbf{75.50 ± 0.15} & \underline{69.18 ± 1.98} & 77.84 ± 1.76          \\
\multicolumn{1}{c|}{}                              & \multicolumn{1}{l|}{FedNH}    & \underline{69.01 ± 2.51} & \underline{65.02 ± 1.23} & \multicolumn{1}{c|}{\textbf{84.63 ± 2.11}} & \underline{75.34 ± 0.86} & \textbf{69.64 ± 1.15} & \textbf{79.53 ± 2.14} \\ \midrule
\multicolumn{1}{c|}{\multirow{9}{*}{Cifar100}}     & \multicolumn{1}{l|}{Local}    & ---                   & 13.63 ± 2.45          & \multicolumn{1}{c|}{30.89 ± 1.82}          & ---                   & 9.44 ± 1.27           & 16.71 ± 1.03          \\
\multicolumn{1}{c|}{}                              & \multicolumn{1}{l|}{FedAvg}   & \underline{35.14 ± 0.48} & \underline{31.85 ± 1.33} & \multicolumn{1}{c|}{\underline{50.77 ± 2.31}} & \underline{36.07 ± 0.41} & \underline{28.86 ± 1.23} & \underline{38.35 ± 2.11} \\
\multicolumn{1}{c|}{}                              & \multicolumn{1}{l|}{FedPer}   & 15.04 ± 0.06          & 16.15 ± 2.34          & \multicolumn{1}{c|}{33.10 ± 1.50}          & 14.69 ± 0.03          & 11.61 ± 2.17          & 19.08 ± 1.36          \\
\multicolumn{1}{c|}{}                              & \multicolumn{1}{l|}{Ditto}    & 35.14 ± 0.48          & 26.19 ± 1.11          & \multicolumn{1}{c|}{45.91 ± 2.17}          & 36.07 ± 0.41          & 22.92 ± 1.77          & 32.81 ± 2.16          \\
\multicolumn{1}{c|}{}                              & \multicolumn{1}{l|}{FedRep}   & 5.42 ± 0.03           & 13.59 ± 2.31          & \multicolumn{1}{c|}{29.45 ± 2.45}          & 6.37 ± 0.04           & 9.47 ± 2.27           & 16.07 ± 1.27          \\
\multicolumn{1}{c|}{}                              & \multicolumn{1}{l|}{FedProto} & ---                   & 10.64 ± 1.02          & \multicolumn{1}{c|}{19.11 ± 1.75}          & ---                   & 9.24 ± 1.33           & 12.61 ± 1.78          \\
\multicolumn{1}{c|}{}                              & \multicolumn{1}{l|}{CReFF}    & \underline{22.90 ± 0.30} & \underline{31.85 ± 1.33} & \multicolumn{1}{c|}{\underline{50.77 ± 2.31}} & \underline{22.21 ± 0.15} & \underline{28.86 ± 1.23} & \underline{38.35 ± 2.11} \\
\multicolumn{1}{c|}{}                              & \multicolumn{1}{l|}{FedBABU}  & 32.41 ± 0.40          & 28.96 ± 2.16          & \multicolumn{1}{c|}{47.86 ± 1.03}          & 32.34 ± 0.49          & 25.84 ± 1.44          & 34.85 ± 1.80          \\
\multicolumn{1}{c|}{}                              & \multicolumn{1}{l|}{FedROD}   & 33.83 ± 0.25          & 28.53 ± 1.27          & \multicolumn{1}{c|}{42.93 ± 1.03}          & 35.20 ± 0.19          & 27.58 ± 1.98          & 33.44 ± 1.76          \\
\multicolumn{1}{c|}{}                              & \multicolumn{1}{l|}{FedNH}    & \textbf{41.34 ± 0.25} & \textbf{38.25 ± 1.23} & \multicolumn{1}{c|}{\textbf{55.21 ± 2.11}} & \textbf{43.19 ± 0.24} & \textbf{36.88 ± 1.15} & \textbf{45.46 ± 2.14} \\ \midrule
\multicolumn{1}{c|}{\multirow{9}{*}{TinyImageNet}} & \multicolumn{1}{l|}{Local}    & ---                   & 7.55 ± 2.45           & \multicolumn{1}{c|}{19.94 ± 1.82}          & ---                   & 5.10 ± 1.27           & 9.93 ± 1.03           \\
\multicolumn{1}{c|}{}                              & \multicolumn{1}{l|}{FedAvg}   & 34.63 ± 0.26          & 27.35 ± 1.33          & \multicolumn{1}{c|}{44.97 ± 2.31}          & 37.65 ± 0.37          & 28.82 ± 1.23          & 37.15 ± 2.11          \\
\multicolumn{1}{c|}{}                              & \multicolumn{1}{l|}{FedPer}   & 15.28 ± 0.14          & 13.84 ± 2.34          & \multicolumn{1}{c|}{30.72 ± 1.50}          & 13.71 ± 0.07          & 9.82 ± 2.17           & 17.05 ± 1.36          \\
\multicolumn{1}{c|}{}                              & \multicolumn{1}{l|}{Ditto}    & 34.63 ± 0.26          & 23.85 ± 1.11          & \multicolumn{1}{c|}{42.67 ± 2.17}          & 37.65 ± 0.37          & 24.97 ± 1.77          & 34.70 ± 2.16          \\
\multicolumn{1}{c|}{}                              & \multicolumn{1}{l|}{FedRep}   & 3.27 ± 0.02           & 9.24 ± 2.31           & \multicolumn{1}{c|}{22.86 ± 2.45}          & 3.91 ± 0.03           & 5.76 ± 2.27           & 10.86 ± 1.27          \\
\multicolumn{1}{c|}{}                              & \multicolumn{1}{l|}{FedProto} & ---                   & 5.17 ± 1.02           & \multicolumn{1}{c|}{10.44 ± 1.75}          & ---                    & 4.21 ± 1.33           & 6.34 ± 1.78           \\
\multicolumn{1}{c|}{}                              & \multicolumn{1}{l|}{CReFF}    & 25.82 ± 0.41          & 27.35 ± 1.33          & \multicolumn{1}{c|}{44.97 ± 2.31}          & 27.87 ± 0.38          & 28.82 ± 1.23          & 37.15 ± 2.11          \\
\multicolumn{1}{c|}{}                              & \multicolumn{1}{l|}{FedBABU}  & 26.36 ± 0.32          & 20.85 ± 2.16          & \multicolumn{1}{c|}{37.96 ± 1.03}          & 30.25 ± 0.32          & 22.74 ± 1.44          & 31.01 ± 1.80          \\
\multicolumn{1}{c|}{}                              & \multicolumn{1}{l|}{FedROD}   & \underline{36.46 ± 0.28} & \underline{28.23 ± 1.27} & \multicolumn{1}{c|}{\underline{45.26 ± 1.03}} & \underline{37.71 ± 0.31} & \underline{29.65 ± 1.98} & \textbf{38.43 ± 1.76} \\
\multicolumn{1}{c|}{}                              & \multicolumn{1}{l|}{FedNH}    & \textbf{36.71 ± 0.36} & \textbf{30.99 ± 1.23} & \multicolumn{1}{c|}{\textbf{46.14 ± 2.11}} & \textbf{38.68 ± 0.30} & \textbf{30.58 ± 1.15} & \underline{38.25 ± 2.14} \\ \bottomrule
\end{tabular}
\caption{Comparison of testing accuracy. The best results are in bold font while the second best results are underlined. The lines ``---'' represent results are not available. The numbers (mean ± std) are the average of three independent runs.}
\label{tab:acc}
\end{table*}
%

\section{Experiments}\label{sec:experiments}
\subsection{Setups, Evaluation, and Baselines}
\textbf{Setups.}~~Extensive experiments are conducted on Cifar10, Cifar100
~\cite{krizhevsky2009learning}, 
and TinyImageNet
three popular benchmark datasets. We follow~\cite{chen2021bridging} to use a simple Convolutional neural network for both Cifar10 and Cifar100 datasets while we use Resnet18~\cite{he2016deep} for the TinyImageNet. To simulate the heterogeneity with class imbalance, we follow~\cite{lin2020ensemble} to distribute each class to clients using the Dirichlet($\beta$) distribution with $\beta\in\{0.3, 1.0\}$
, resulting in clients having different class distributions and different number of samples. Note that when $\beta\leq 1.0$, each client is likely to have one or two dominating classes while owning a few or even zero samples from the remaining classes. Consequently, both classes and the number of samples are imbalanced among clients. To simulate the cross-device setting, we consider 100 clients with a 10\% participation ratio.

\noindent\textbf{Evaluation Metrics.}~~ We evaluate the testing accuracy of both global and personalized models. We take the latest local models as the personalized model for methods that do not explicitly produce personalized models. For PFL methods that do not upload the head to the server, we use the head initialized at the starting point.
We follow~\cite{chen2021bridging} to evaluate personalized models on a single testing \textit{balanced} dataset $\Dcal^{\text{test}}$ to reduce the randomness from dataset partitioning. Specifically, the accuracy of the $i$th personalized model is computed as 
$$
\text{acc}_i = \frac{\sum_{(x_j,y_j)\sim \Dcal^{\text{test}} } \alpha_{i}\left(y_{j}\right) \mathbf{1}\left(y_{j}=\hat{y}_{j}\right)}{\sum_{(x_j,y_j)\sim \Dcal^{\text{test}}} \alpha_{i}\left(y_{j}\right)},
$$
where $\alpha_i(\cdot)$ is a positive-valued function and $\mathbf{1(\cdot)}$ is the indicator function. $y_j$ and $\hat y_j$ are the true label and predicted label, respectively. We consider two choices of setting $\alpha_i(\cdot)$. The first choice is to set $\alpha_i(y) = \Pmbb_i(y=c)$, where $\Pmbb_i(y=c)$ stands for the probability that the sample $y$ is from class $c$ in the $i$th client. The probability $\Pmbb_i(y=c)$ can be estimated from the $i$th client's training dataset. The second choice sets $\alpha_i(y)$ to $1$ if the class $y$ appears in $i$th client's training dataset and $0$ otherwise. This metric measures the generalization ability of personalized models because it treats all classes presented in the local training dataset with equal importance.

\noindent\textbf{Baselines.}~~We choose several popular and state-of-the-art FL/PFL algorithms, such as FedAvg~\cite{mcmahan2017communication}, FedPer~\cite{arivazhagan2019federated}, Ditto~\cite{li2021ditto}, FedRep~\cite{collins2021exploiting},
FedProto~\cite{tan2022fedproto}, FedBABU~\cite{oh2021fedbabu}, FedROD~\cite{chen2021bridging}, and CReFF~\cite{shang2022federated}. 

Experimental environments and implementation details on all chosen methods are deferred to 
Appendix~\ref{appendix:experiments}\footnote{Our code is publicly available at \url{https://github.com/Yutong-Dai/FedNH}.}.

\begin{figure}[t]
     \centering
     \begin{subfigure}[b]{0.48\textwidth}
         \centering
         \includegraphics[width=5cm]{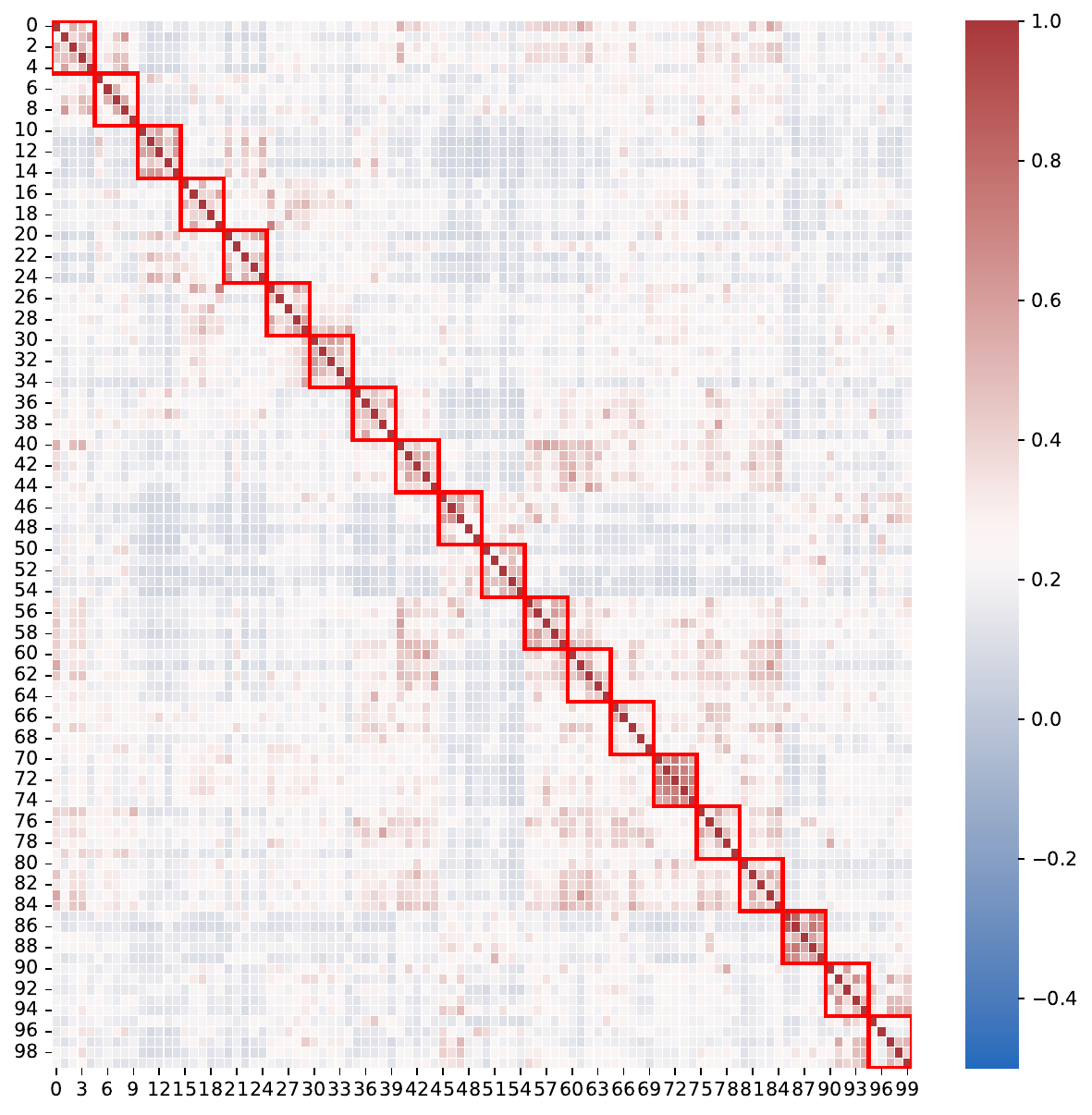}
     \end{subfigure}
     \hfill
     \begin{subfigure}[b]{0.48\textwidth}
         \centering
         \includegraphics[width=5cm]{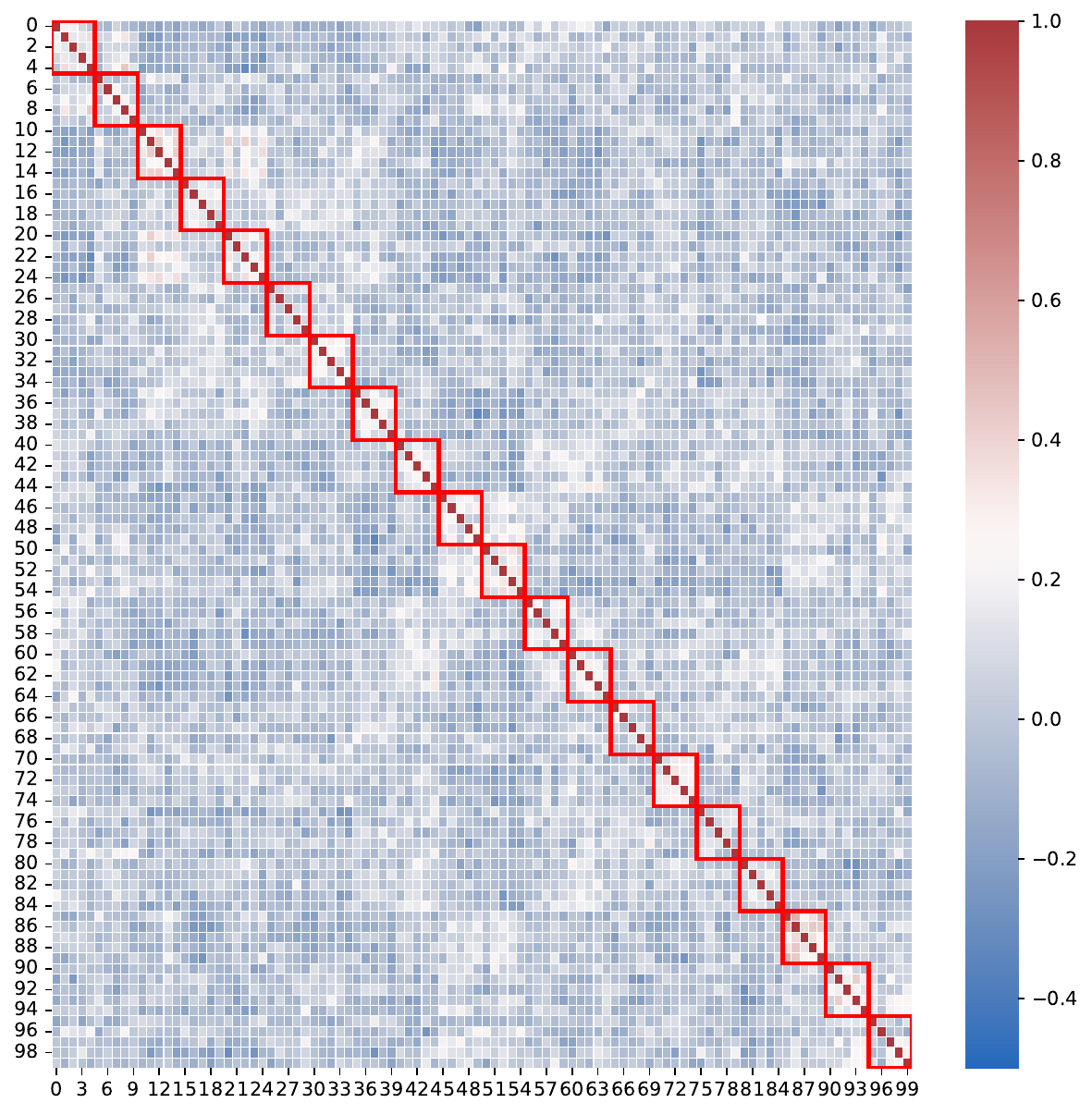}
     \end{subfigure}
        \caption{Similarity of prototypes learned with  FedNH (left) and FedAvg (right).}
        \label{fig:sim_fednh_avg}
\end{figure}

\subsection{Results}
We report three types of testing accuracy in Table~\ref{tab:acc}. Numbers under the GM column are the averaged testing accuracy for the global model (mean$\pm$standard deviation); PM(V) reports the averaged testing accuracy of personalized models by assigning equal weights to all appearing classes (the second choice of setting the weight function $\alpha_i(\cdot)$); PM(L) represents the averaged testing accuracy by assuming the training and testing dataset has the same distribution (the first choice of setting the weight function $\alpha_i(\cdot)$). Several observations are made:
\begin{itemize}
    \item Personalized models from FedAvg serve as the strong baseline (evaluated under both PM(L) and PM(v) metrics) and can outperform many specialized personalized methods. Similar observations are also made in~\cite{oh2021fedbabu,chen2021bridging}.
    \item In the cross-device setting, the performance gain of the personalized models from FedNH are the most consistent across all tested cases. Further, FedNH has the best PM(V) testing accuracy in almost all cases. Moreover, it adds minimum computation overhead compared with state-of-the-art methods like FedROD. 
    \item The PM(V) metric is less sensitive to the class imbalance than PM(L). Note that the accuracy of PM(L) changes significantly from the Dir(0.3) to Dir(1.0), while PM(L) is relatively stable, as shown in Table~\ref{tab:acc}.
    \item A strong global model often leads to strong personalized models by comparing the GM and PM metrics.
\end{itemize}
Discussions on why some methods like FedProto and CReFF do not perform well are discussed in Appendix~\ref{appendix:accuracy-discussions}.

\subsection{Analysis}~\label{sec:exp.analysis}
\textbf{Capture of Class Semantics.}
To further validate that prototypes learned by FedNH can capture the class semantics, we visualize the pairwise class prototypes' similarity from Cifar100 in Figure~\ref{fig:sim_fednh_avg}. 100 classes in Cifar100 form 20 super-classes.
In Figure~\ref{fig:sim_fednh_avg}, we group classes from one super-class into one red box along the main diagonal. Because the classes within one red box are semantically closer, their prototypes should also be more similar. We can see that our FedNH learns the semantics of classes by capturing their fine-grained similarities while the FedAvg simply treats all classes as different. Similar plots for other methods are given in Appendix~\ref{appendix:experiments} 
with a more detailed analysis. We also visualize learned representations in  Figure~\ref{fig:local.cifar100} Appendix~\ref{appendix:experiments}) 
for different classes on different clients and find that similar classes' representations tend to be more close.

\textbf{Sensitivity Analysis.}
We conduct the sensitivity analysis on the smoothing parameter $\rho$ by choosing it from $\{0.1, 0.3, 0.5, 0.7, 0.9\}$ and plotting the relative performance gain over the base case $\rho=0.1$. As shown in Figure~\ref{fig:sensitivity}, $\rho=0.1$ gives the worst performance. This matches Theorem~\ref{thm:convergence} that suggests $\rho$ cannot be too small.
\begin{figure}[!th]
    \centering
    \includegraphics[width=0.4\textwidth]{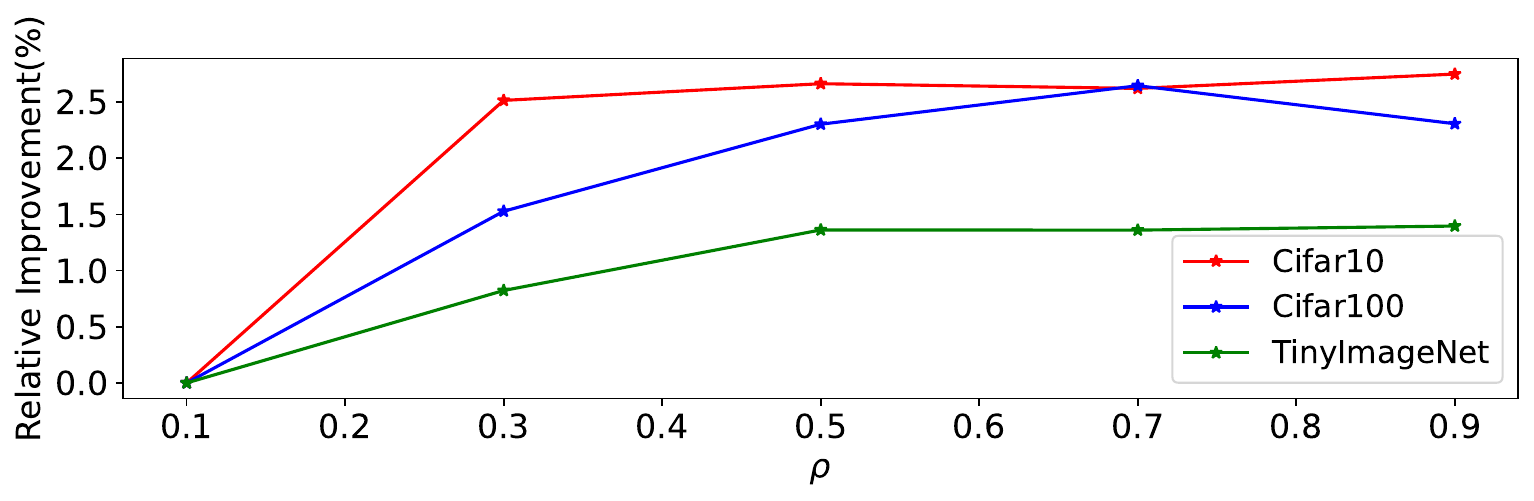}
    \caption{Sensitivity analysis on the smoothing parameter $\rho$.}
    \label{fig:sensitivity}
\end{figure}

\textbf{Fairness Analysis.}
We measure fairness by computing the standard deviation of the accuracy of personalized models across all clients. A smaller standard deviation indicates that all clients' performance tend to be concentrated around the mean accuracy. In Table~\ref{tab:fairness}, we present the result in the form of mean accuracy ± standard deviation. One can see that FedNH improves all clients more equally.
\begin{table}[!th]
    \centering
    \resizebox{0.47\textwidth}{!}{
    \begin{tabular}{l|lll}
    \toprule
     & Cifar10 & Cifar100 & TinyImageNet \\
    \midrule
    Local & 71.57 ± 10.13 & 30.89 ± 4.60 & 19.94 ± 3.27 \\
    FedAvg & 84.08 ± 6.80 & 50.77 ± 4.23 & 44.97 ± 2.99 \\
    FedPer & 82.38 ± 6.38 & 33.10 ± 4.26 & 30.72 ± 3.21 \\
    Ditto & 80.08 ± 7.83 & 45.91 ± 4.10 & 42.67 ± 3.05 \\
    FedRep & 80.22 ± 7.06 & 29.45 ± 4.19 & 22.86 ± 3.20 \\
    FedProto & 68.35 ± 11.03 & 19.11 ± 4.17 & 10.44 ± 2.80 \\
    FedBABU & 82.64 ± 6.11 & 47.86 ± 3.89 & 37.96 ± 2.79 \\
    FedROD & 83.44 ± 5.89 & 42.93 ± 4.10 & 45.26 ± 2.72 \\
    FedNH & 84.63 ± \textbf{5.61} & 55.21 ± \textbf{3.91} & 46.14 ±\textbf{ 2.70} \\
    \bottomrule
    \end{tabular}
    }
    \caption{Compare the fairness of baseline methods.}
    \label{tab:fairness}
\end{table}
\section{Conclusion}
In this work, we proposed FedNH, a novel FL algorithm to address the data heterogeneity with class imbalance. FedNH combines the uniformity and semantics of class prototypes to learn high-quality representations for classification. Extensive experiments were conducted to show the effectiveness and robustness of our method over recent works.

\textbf{Limitation and Future Work.} Our idea currently only applies to the classification task, and the inductive bias from uniformity and semantics of class prototypes can only be imposed on the head of neural network architecture. Our future work will explore the possibility of extending the inductive bias to the intermediate layers of a neural network and different vision tasks.

\section*{Acknowledgement.} This work was partially supported by the National Science Foundation Grants CRII-2246067.

\bibliographystyle{unsrt}
\bibliography{ref.bib}

\begin{thebibliography}{10}

\bibitem{mcmahan2017communication}
Brendan McMahan, Eider Moore, Daniel Ramage, Seth Hampson, and Blaise~Aguera
  y~Arcas.
\newblock Communication-efficient learning of deep networks from decentralized
  data.
\newblock In {\em Artificial intelligence and statistics}, pages 1273--1282.
  PMLR, 2017.

\bibitem{tan2022towards}
Alysa~Ziying Tan, Han Yu, Lizhen Cui, and Qiang Yang.
\newblock Towards personalized federated learning.
\newblock {\em IEEE Transactions on Neural Networks and Learning Systems},
  2022.

\bibitem{ng2021federated}
Dianwen Ng, Xiang Lan, Melissa Min-Szu Yao, Wing~P Chan, and Mengling Feng.
\newblock Federated learning: a collaborative effort to achieve better medical
  imaging models for individual sites that have small labelled datasets.
\newblock {\em Quantitative Imaging in Medicine and Surgery}, 11(2):852, 2021.

\bibitem{kang2019decoupling}
Bingyi Kang, Saining Xie, Marcus Rohrbach, Zhicheng Yan, Albert Gordo, Jiashi
  Feng, and Yannis Kalantidis.
\newblock Decoupling representation and classifier for long-tailed recognition.
\newblock In {\em International Conference on Learning Representations}, 2019.

\bibitem{zhou2020bbn}
Boyan Zhou, Quan Cui, Xiu-Shen Wei, and Zhao-Min Chen.
\newblock Bbn: Bilateral-branch network with cumulative learning for
  long-tailed visual recognition.
\newblock In {\em Proceedings of the IEEE/CVF conference on computer vision and
  pattern recognition}, pages 9719--9728, 2020.

\bibitem{oh2021fedbabu}
Jaehoon Oh, SangMook Kim, and Se-Young Yun.
\newblock Fedbabu: Toward enhanced representation for federated image
  classification.
\newblock In {\em International Conference on Learning Representations}, 2021.

\bibitem{chen2021bridging}
Hong-You Chen and Wei-Lun Chao.
\newblock On bridging generic and personalized federated learning for image
  classification.
\newblock In {\em International Conference on Learning Representations}, 2021.

\bibitem{arivazhagan2019federated}
Manoj~Ghuhan Arivazhagan, Vinay Aggarwal, Aaditya~Kumar Singh, and Sunav
  Choudhary.
\newblock Federated learning with personalization layers.
\newblock {\em arXiv preprint arXiv:1912.00818}, 2019.

\bibitem{collins2021exploiting}
Liam Collins, Hamed Hassani, Aryan Mokhtari, and Sanjay Shakkottai.
\newblock Exploiting shared representations for personalized federated
  learning.
\newblock In {\em International Conference on Machine Learning}, pages
  2089--2099. PMLR, 2021.

\bibitem{tan2022fedproto}
Yue Tan, Guodong Long, Lu~Liu, Tianyi Zhou, Qinghua Lu, Jing Jiang, and Chengqi
  Zhang.
\newblock Fedproto: Federated prototype learning across heterogeneous clients.
\newblock In {\em AAAI Conference on Artificial Intelligence}, volume~1,
  page~3, 2022.

\bibitem{kubat1997addressing}
Miroslav Kubat and Stan Matwin.
\newblock Addressing the curse of imbalanced training sets: One-sided
  selection.
\newblock In {\em ICML}, 1997.

\bibitem{chawla2002smote}
Nitesh~V Chawla, Kevin~W Bowyer, Lawrence~O Hall, and W~Philip Kegelmeyer.
\newblock Smote: synthetic minority over-sampling technique.
\newblock {\em Journal of artificial intelligence research}, 16:321--357, 2002.

\bibitem{he2009learning}
Haibo He and Edwardo~A Garcia.
\newblock Learning from imbalanced data.
\newblock {\em IEEE Transactions on knowledge and data engineering},
  21(9):1263--1284, 2009.

\bibitem{lin2017focal}
Tsung-Yi Lin, Priya Goyal, Ross Girshick, Kaiming He, and Piotr Doll{\'a}r.
\newblock Focal loss for dense object detection.
\newblock In {\em Proceedings of the IEEE international conference on computer
  vision}, pages 2980--2988, 2017.

\bibitem{khan2017cost}
Salman~H Khan, Munawar Hayat, Mohammed Bennamoun, Ferdous~A Sohel, and Roberto
  Togneri.
\newblock Cost-sensitive learning of deep feature representations from
  imbalanced data.
\newblock {\em IEEE transactions on neural networks and learning systems},
  29(8):3573--3587, 2017.

\bibitem{cao2019learning}
Kaidi Cao, Colin Wei, Adrien Gaidon, Nikos Arechiga, and Tengyu Ma.
\newblock Learning imbalanced datasets with label-distribution-aware margin
  loss.
\newblock {\em Advances in neural information processing systems}, 32, 2019.

\bibitem{cui2019class}
Yin Cui, Menglin Jia, Tsung-Yi Lin, Yang Song, and Serge Belongie.
\newblock Class-balanced loss based on effective number of samples.
\newblock In {\em Proceedings of the IEEE/CVF conference on computer vision and
  pattern recognition}, pages 9268--9277, 2019.

\bibitem{wang2020long}
Xudong Wang, Long Lian, Zhongqi Miao, Ziwei Liu, and Stella Yu.
\newblock Long-tailed recognition by routing diverse distribution-aware
  experts.
\newblock In {\em International Conference on Learning Representations}, 2020.

\bibitem{duan2020self}
Moming Duan, Duo Liu, Xianzhang Chen, Renping Liu, Yujuan Tan, and Liang Liang.
\newblock Self-balancing federated learning with global imbalanced data in
  mobile systems.
\newblock {\em IEEE Transactions on Parallel and Distributed Systems},
  32(1):59--71, 2020.

\bibitem{wang2021addressing}
Lixu Wang, Shichao Xu, Xiao Wang, and Qi~Zhu.
\newblock Addressing class imbalance in federated learning.
\newblock In {\em Proceedings of the AAAI Conference on Artificial
  Intelligence}, volume~35, pages 10165--10173, 2021.

\bibitem{shang2022federated}
Xinyi Shang, Yang Lu, Gang Huang, and Hanzi Wang.
\newblock Federated learning on heterogeneous and long-tailed data via
  classifier re-training with federated features.
\newblock {\em arXiv preprint arXiv:2204.13399}, 2022.

\bibitem{agarap2018deep}
Abien~Fred Agarap.
\newblock Deep learning using rectified linear units (relu).
\newblock {\em arXiv preprint arXiv:1803.08375}, 2018.

\bibitem{li2022targeted}
Tianhong Li, Peng Cao, Yuan Yuan, Lijie Fan, Yuzhe Yang, Rogerio~S Feris, Piotr
  Indyk, and Dina Katabi.
\newblock Targeted supervised contrastive learning for long-tailed recognition.
\newblock In {\em Proceedings of the IEEE/CVF Conference on Computer Vision and
  Pattern Recognition}, pages 6918--6928, 2022.

\bibitem{su2012improving}
Yu~Su and Fr{\'e}d{\'e}ric Jurie.
\newblock Improving image classification using semantic attributes.
\newblock {\em International journal of computer vision}, 100(1):59--77, 2012.

\bibitem{he2015delving}
Kaiming He, Xiangyu Zhang, Shaoqing Ren, and Jian Sun.
\newblock Delving deep into rectifiers: Surpassing human-level performance on
  imagenet classification.
\newblock In {\em Proceedings of the IEEE international conference on computer
  vision}, pages 1026--1034, 2015.

\bibitem{glorot2010understanding}
Xavier Glorot and Yoshua Bengio.
\newblock Understanding the difficulty of training deep feedforward neural
  networks.
\newblock In {\em Proceedings of the thirteenth international conference on
  artificial intelligence and statistics}, pages 249--256. JMLR Workshop and
  Conference Proceedings, 2010.

\bibitem{wright1999numerical}
Stephen Wright, Jorge Nocedal, et~al.
\newblock Numerical optimization.
\newblock {\em Springer Science}, 35(67-68):7, 1999.

\bibitem{wang2017normface}
Feng Wang, Xiang Xiang, Jian Cheng, and Alan~Loddon Yuille.
\newblock Normface: L2 hypersphere embedding for face verification.
\newblock In {\em Proceedings of the 25th ACM international conference on
  Multimedia}, pages 1041--1049, 2017.

\bibitem{wang2018additive}
Jian Cheng, Feng Wang, Weiyang Liu, and Haijun Liu.
\newblock Additive margin softmax for face verification.
\newblock {\em IEEE Signal Processing Letters}, 25(7):926--930, 2018.

\bibitem{hinton2015distilling}
Geoffrey Hinton, Oriol Vinyals, Jeff Dean, et~al.
\newblock Distilling the knowledge in a neural network.
\newblock {\em arXiv preprint arXiv:1503.02531}, 2(7), 2015.

\bibitem{li2020federated}
Tian Li, Anit~Kumar Sahu, Manzil Zaheer, Maziar Sanjabi, Ameet Talwalkar, and
  Virginia Smith.
\newblock Federated optimization in heterogeneous networks.
\newblock {\em Proceedings of Machine Learning and Systems}, 2:429--450, 2020.

\bibitem{karimireddy2020scaffold}
Sai~Praneeth Karimireddy, Satyen Kale, Mehryar Mohri, Sashank Reddi, Sebastian
  Stich, and Ananda~Theertha Suresh.
\newblock Scaffold: Stochastic controlled averaging for federated learning.
\newblock In {\em International Conference on Machine Learning}, pages
  5132--5143. PMLR, 2020.

\bibitem{acar2020federated}
Durmus Alp~Emre Acar, Yue Zhao, Ramon Matas, Matthew Mattina, Paul Whatmough,
  and Venkatesh Saligrama.
\newblock Federated learning based on dynamic regularization.
\newblock In {\em International Conference on Learning Representations}, 2020.

\bibitem{stich2019local}
Sebastian~Urban Stich.
\newblock Local sgd converges fast and communicates little.
\newblock In {\em ICLR 2019-International Conference on Learning
  Representations}, number CONF, 2019.

\bibitem{yu2019parallel}
Hao Yu, Sen Yang, and Shenghuo Zhu.
\newblock Parallel restarted sgd with faster convergence and less
  communication: Demystifying why model averaging works for deep learning.
\newblock In {\em Proceedings of the AAAI Conference on Artificial
  Intelligence}, volume~33, pages 5693--5700, 2019.

\bibitem{krizhevsky2009learning}
Alex Krizhevsky.
\newblock Learning multiple layers of features from tiny images.
\newblock {\em MSc thesis}, 2009.

\bibitem{he2016deep}
Kaiming He, Xiangyu Zhang, Shaoqing Ren, and Jian Sun.
\newblock Deep residual learning for image recognition.
\newblock In {\em Proceedings of the IEEE conference on computer vision and
  pattern recognition}, pages 770--778, 2016.

\bibitem{lin2020ensemble}
Tao Lin, Lingjing Kong, Sebastian~U Stich, and Martin Jaggi.
\newblock Ensemble distillation for robust model fusion in federated learning.
\newblock {\em Advances in Neural Information Processing Systems},
  33:2351--2363, 2020.

\bibitem{li2021ditto}
Tian Li, Shengyuan Hu, Ahmad Beirami, and Virginia Smith.
\newblock Ditto: Fair and robust federated learning through personalization.
\newblock In {\em International Conference on Machine Learning}, pages
  6357--6368. PMLR, 2021.

\bibitem{PyTorch}
Adam Paszke, Sam Gross, Francisco Massa, Adam Lerer, James Bradbury, Gregory
  Chanan, Trevor Killeen, Zeming Lin, Natalia Gimelshein, Luca Antiga, Alban
  Desmaison, Andreas Kopf, Edward Yang, Zachary DeVito, Martin Raison, Alykhan
  Tejani, Sasank Chilamkurthy, Benoit Steiner, Lu~Fang, Junjie Bai, and Soumith
  Chintala.
\newblock Pytorch: An imperative style, high-performance deep learning library.
\newblock In H.~Wallach, H.~Larochelle, A.~Beygelzimer, F.~d\textquotesingle
  Alch\'{e}-Buc, E.~Fox, and R.~Garnett, editors, {\em Advances in Neural
  Information Processing Systems 32}, pages 8024--8035. Curran Associates,
  Inc., 2019.

\end{thebibliography}

\newpage
\appendix
\onecolumn
\section{Experiments}\label{appendix:experiments}

\subsection{Description of Datasets}\label{appendix:datasets.}
Cifar10 contains $10$ classes of images, and each class has $5000$ images in the training dataset and $1000$ images in the testing dataset, respectively. Similarly, Cifar100 contains $100$ classes of images, and each class has $500$ images in the training dataset and $100$ images in the testing dataset, respectively.
 TinyImageNet contains 200 image classes, and each class has $500$ and $50$ images in the training and testing datasets, respectively. Images are of resolution $64\times 64$.

\subsection{Implementation Details}\label{appendix:impl.details}

We use Pytorch~\cite{PyTorch} to implement all algorithms. SGD optimizer is used with a 0.01 initial learning rate and 0.9 momentum for all local training. We decay the learning rate exponentially with a factor of 0.99 in each round. We apply $10^{-5}$ weight decay for Cifar 10 and Cifar 100 while $10^{-3}$ for TinyImageNet to prevent overfitting. In each round, all methods locally train for 5 epochs, and the batch size is set to 64. All methods are trained for 200 rounds, after which there is no significant gain over the testing accuracy. All experiments are run a Linux (Ubuntu) server with 4 A100-SXM4-40GB GPUs and 24 Intel(R) Xeon(R) CPUs. 

We developed a unified code base so all baseline methods can be fairly tested in the same environment. For example, since we fix the random seed, it is guaranteed that all baseline methods are tested under exactly the same clients' configuration, including the starting point (if no special steps are performed in the initialization stage), local training dataset $\Dcal_k$, and the precise rounds that are selected to participate training. More details can be found in the source code.

\subsubsection{Model Architecture.} The CNN model used for Cifar10 / Cifar 100 is a 5-layer neural network with 2 Convolutional layers and 3 linear layers. A Pytorch-like code is given in Listing~\ref{lst:cnn}.
\begin{lstlisting}[language=Python, caption=CNN architecture, label={lst:cnn},captionpos=b]
class Conv2Cifar(Model):
    def __init__(self, num_classes):
        super().__init__()
        self.conv1 = nn.Conv2d(in_channels=3, out_channels=64, kernel_size=5)
        self.conv2 = nn.Conv2d(in_channels=64, out_channels=64, kernel_size=5)
        self.pool = nn.MaxPool2d(kernel_size=2, stride=2)
        self.linear1 = nn.Linear(64 * 5 * 5, 384)
        self.linear2 = nn.Linear(384, 192)
        self.linear3 = nn.Linear(192, num_classes, bias=False)

    def forward(self, x):
        x = self.pool(F.relu(self.conv1(x)))
        x = self.pool(F.relu(self.conv2(x)))
        x = x.view(-1, 64 * 5 * 5)
        x = F.relu(self.linear1(x))
        x = F.relu(self.linear2(x))
        logits = self.linear3(x)
        return logits
\end{lstlisting}

For Resnet18, we follow~\cite{lin2020ensemble} to replace batch normalization with group normalization, and we set the number of groups to 2. The rest part of the ResNet18 remains unchanged.

\subsubsection{Choices of Hyper-Parameters and Additional Comments on Implementing Baseline Methods.}\label{appendix:hyperparams}

\begin{itemize}
    \item For the local training method, each client is trained for exactly the same number of rounds as its counterpart in FedAvg. For example, if client 0 is selected at rounds from \{1, 10, 56, 101, 193\}, a total of 5 rounds, then client 0 in the local training setting will be precisely trained for 5 rounds.
    \item For FedPer, one needs to specify which layers are to share and which layers are kept private. As stated in the Section~ \ref{sec:introduction}, we take the last linear layer as the personalized part while sharing the rest.
    \item For Ditto, the penalty parameter $\lambda$ for regularization term is set to 0.75 as chosen in \cite{collins2021exploiting, oh2021fedbabu}. The optimal global model is estimated using the FedAvg as one option suggested in \cite{li2021ditto}.
    \item For FedRep, we follow the official implementation\footnote{https://github.com/lgcollins/FedRep} that we first train the head for 10 epochs with the body fixed, and then train the body for 5 epochs with the head fixed.
    
    \item For FedProto, we mainly follow the official implementation\footnote{https://github.com/yuetan031/fedproto} with some modifications to improve the efficiency. The penalty parameter $\lambda$ is set to 0.1 as is suggested in \cite{tan2022fedproto}.
    \item For CReFF, we follow the official implementation \footnote{https://github.com/shangxinyi/CReFF-FL}. The number of federated features is set to 100 per class. The number of epochs of optimizing federated features is set to 100, while the number of epochs to re-train the classifier is set to 300 for each round. We remark that the learning rate for federated features is suggested to be 0.1, and the learning rate for the classifier in the server is 0.01 in the official implementation. However, in the cross-device setting, such a choice makes the training unstable, and the global model fails to converge. We conjecture the cause is the incompatible learning rate for the classifier (0.01) and the federated features (0.1). So we set both values of the learning rate to 0.01.
    \item For FedBABU, our implementation reference from the official implementation\footnote{https://github.com/jhoon-oh/FedBABU}. The head is initialized using Pytorch's default implementation\footnote{https://github.com/pytorch/pytorch/blob/684a404defaaa9e55ac905a03d9f6a368f3c45f7/torch/nn/modules/linear.py\#L67}. After finishing the training, the local model for each client is fine-tuned for 5 epochs with the head also being updated.
    \item For FedROD, there is no official implementation. Therefore, our implementation is based on the partial code released by the authors on the OpenReview\footnote{https://openreview.net/revisions?id=I1hQbx10Kxn}. It is worth mentioning that during the local training, the global head and the body is trained first for 5 epochs. And then, the personalized head is trained with the body fixed. This training strategy implies that FedROD has more computation cost compared with methods like FedBABU and FedPer. Lastly, for the CNN model, we use the Hyper-network to generate the personalized head while not for the ResNet18. This is due to the GPU memory constraints.
    \item For FedNH, we use $\rho=0.9$ for all tests. Additionally, for the scaling parameter $s$ appeared in Eq.~\eqref{prob:local-prob}, instead of fixing it to a constant, we also make it trainable for more flexibility. The aggregation weight 
    is set as $\alpha_k^t=\frac{1}{|\Scal^t|}$.
\end{itemize}

\subsubsection{Setups for the Spiral example.}
For the results reported in Figure~\ref{fig:b_fh}- Figure~\ref{fig:ub_uh}, we use SGD as the optimizer with 0.1 initial learning rate, 0.9 momentum, and $10^{-5}$ weight decay. We decay the learning rate by 0.99 for every epoch. For the results reported in Figure~\ref{fig:consistent.goal}. We use Dir(0.3) to distribute the balanced spiral dataset to 100 clients to simulate the data heterogeneity and class imbalance. We FedAvg as the learning algorithm.   The client participation rate is 0.1.
Again, we use SGD as the local optimizer with 0.1 initial learning rate, 0.9 momentum, and $10^{-5}$ weight decay. For each round, the selected clients perform 5 epochs of local training. We decay the learning rate by 0.99 for every round.

\subsection{Discussions on testing accuracy}\label{appendix:accuracy-discussions}
One can see that the local models from FedAvg serves as strong personalized models. Similar phenomenon is also reported in \cite{collins2021exploiting,chen2021bridging, oh2021fedbabu}. The local model can be regarded as the result of fine-tuning global, which earns more benefits from federation \cite{oh2021fedbabu}.

 Next, we discuss each baseline mehtod. For FedPer, it assumes the availability of all clients. In the cross-device setting, client sampling might cause performance degradation. For Ditto, the global model is derived the same way as the FedAvg. However, the personalized model is worse than that of FedAvg, which suggests that choosing a proper penalty parameter $\lambda$ could make a big difference.
 The performance drop of FedRep could be attributed to the low epochs of training on the body\cite{oh2021fedbabu}. For FedProto, it assumes the full participation of clients. Therefore, when in the cross-device setting, prototypes proposed by a small set of clients may not capture the global view of the training dataset and may vary rapidly from round to round. Moreover, each round may select clients with drastically different data distributions, hence making the prototypes obtained from the last round less informative to the currently selected clients. For CReFF, the hype-parameters used in \cite{shang2022federated} make the training dynamics unstable, i.e., the global model may not converge in the cross-device setting. After tuning the learning rate for the federated features and the balanced classifier, the global model indeed converges but with the less strong performance as indicated in \cite{shang2022federated} when there are only 20 clients with 40\% participation ratio. 
 
\subsection{Some missing visualization}\label{appendix:miss.vis}
First, the balanced and unbalanced training dataset together with the testing datasets used in Section~\ref{sec:motivating.example} are shown in Figure~\ref{fig:spiral.appendix}. 
\begin{figure}[!h]
     \centering
     \begin{subfigure}[b]{0.33\textwidth}
         \centering
         \includegraphics[width=\textwidth]{./figs/train_spiral_balanced.pdf}
         \caption{Balanced spiral training set}
         \label{fig:appendix-balan-spiral-train}
     \end{subfigure}
     \hfill
     \begin{subfigure}[b]{0.33\textwidth}
         \centering
         \includegraphics[width=\textwidth]{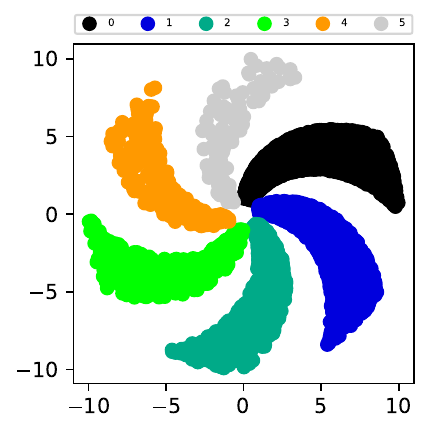}
         \caption{Imbalanced spiral training set}
         \label{fig:appendix-imbalan-spiral-train}
     \end{subfigure}
    \begin{subfigure}[b]{0.33\textwidth}
         \centering
         \includegraphics[width=\textwidth]{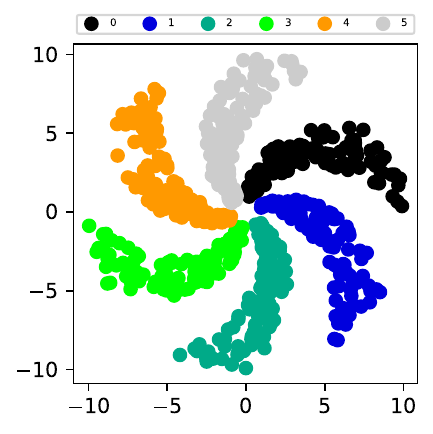}
         \caption{Balanced spiral testing set}
         \label{fig:appendix-spiral-testing}
     \end{subfigure}
    \caption{Visualization of Spiral datasets.}
    \label{fig:spiral.appendix}
\end{figure}

Next, we visualize latent representations learned across different clients as discussed in Section~\ref{sec:exp.analysis} in Figure~\ref{fig:local.cifar100}, where we add a plot to show the class distributions of two clients. Specifically, we choose to compare the representations of three classes \{Oka tree, Leopard, Bear\}. We use PCA to perform dimension reduction and only visualize the first two dimensions. As shown in Figure~\ref{fig:fedavgc0}-Figure~\ref{fig:fednhc1}, the representations learned by FedNH capture more semantics, as those of Leopard and Bear are more close than those of Oka tree. Moreover, the representations of Leopard and Bear are learned more consistently since they occupy the same locations in the latent space across client 0 and client 1. On the contrary, the representations of Leopard and Bear are inconsistent across client 0 and client 1 of FedAvg as they switch locations.
\begin{figure}[!ht]
     \centering
     \begin{subfigure}[b]{0.23\textwidth}
         \centering
         \includegraphics[width=4.5cm]{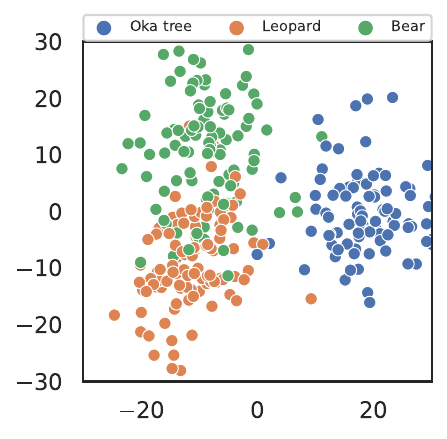}
         \caption{Client 0 of \fedavg{} }
         \label{fig:fedavgc0}
     \end{subfigure}
     \hfill
     \begin{subfigure}[b]{0.23\textwidth}
         \centering
         \includegraphics[width=4.5cm]{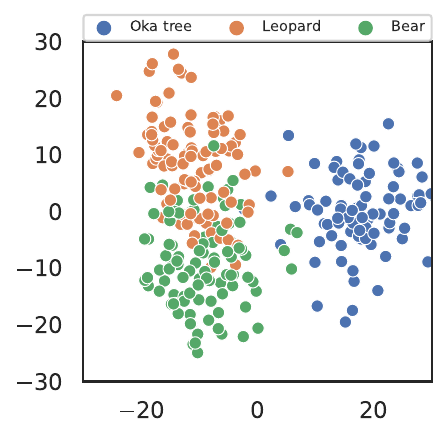}
         \caption{ Client 1 of \fedavg{}}
         \label{fig:fedavgc1}
     \end{subfigure}
    \hfill
    \begin{subfigure}[b]{0.23\textwidth}
         \centering
         \includegraphics[height=4.5cm]{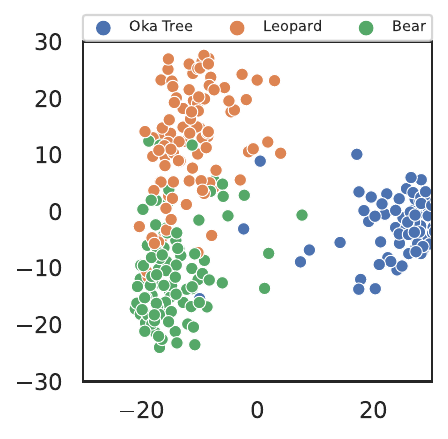}
         \caption{Client 0 of FedNH }
         \label{fig:fednhc0}
    \end{subfigure}  
    \hfill
     \begin{subfigure}[b]{0.23\textwidth}
         \centering
         \includegraphics[height=4.5cm]{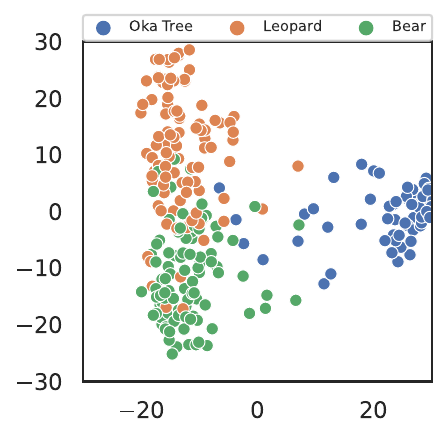}
         \caption{Client 1 of FedNH }
         \label{fig:fednhc1}
     \end{subfigure}     
    \hfill
     \begin{subfigure}[b]{\textwidth}
         \centering
         \includegraphics[height=5cm]{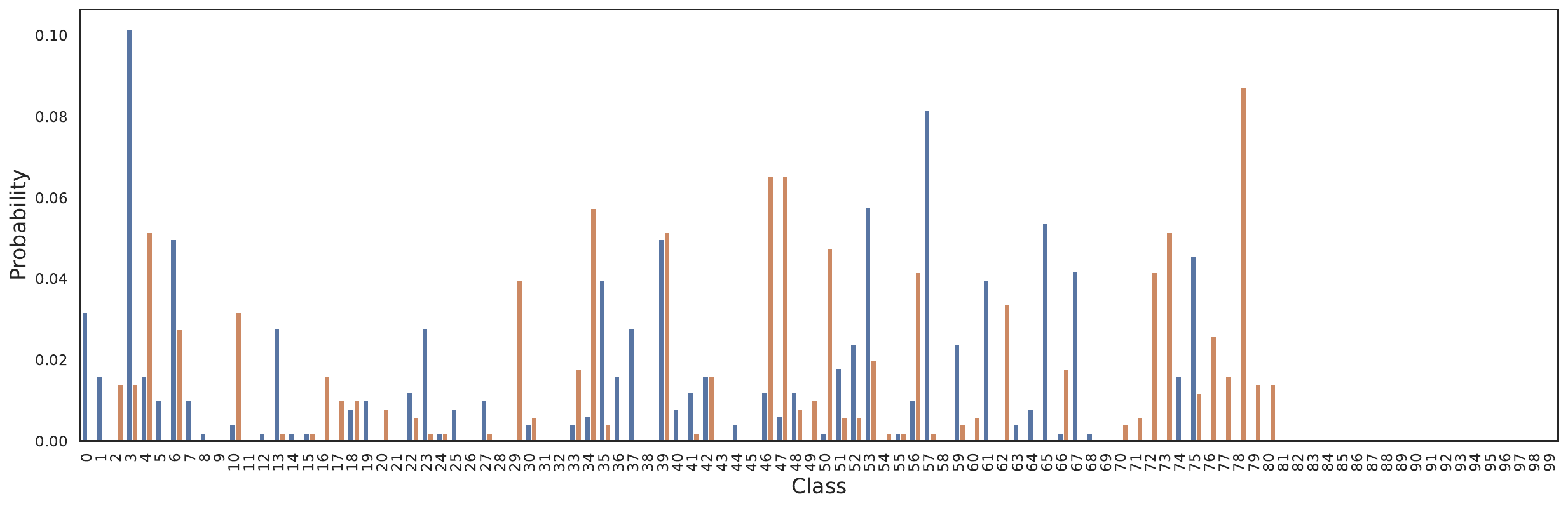}
         \caption{Class distribution for Client 0 and Client 1.}
         \label{fig:client_class_dist. (blue:client 0; orage: client 1)}
     \end{subfigure}         
        \caption{Visualization of local representations from three classes}
        \label{fig:local.cifar100}
\end{figure}

\clearpage
Next, we provide a more detailed analysis of prototypes captured by FedNH in Figure~\ref{fig:semantics-fednh}. To further validate that prototypes learned by FedNH can capture the class semantics, we visualize the pairwise class prototypes' similarity from Cifar100 in Figure~\ref{fig:sim_fednh_avg}. The 100 classes in Cifar100 form 20 super-classes. In Figure~\ref{fig:sim_fednh_avg}, we group classes from one super-class into one red box along the main diagonal. Because the classes within one red box are semantically closer, their prototypes should also be more similar. We can see that our FedNH learns the semantics of classes by capturing their fine-grained similarities. Next, we investigate some blocks off the main diagonal, as highlighted in green boxes. For example, the first green block capture the similarity between \{beaver, dolphin, otter, seal, whale\} and \{aquarium fish, flatfish, ray, shark, trout\}. We can see some are of high correlation, which indeed makes sense. On the contrary, the third green box captures the similarity among super-class
\{bee, beetle, butterfly, caterpillar, cockroach\} and
\{bridge, castle, house, road, skyscraper\}. And clearly, they are less similar. Hence correlations are low.

\begin{figure}[!th]
    \centering
    \includegraphics[scale=0.4]{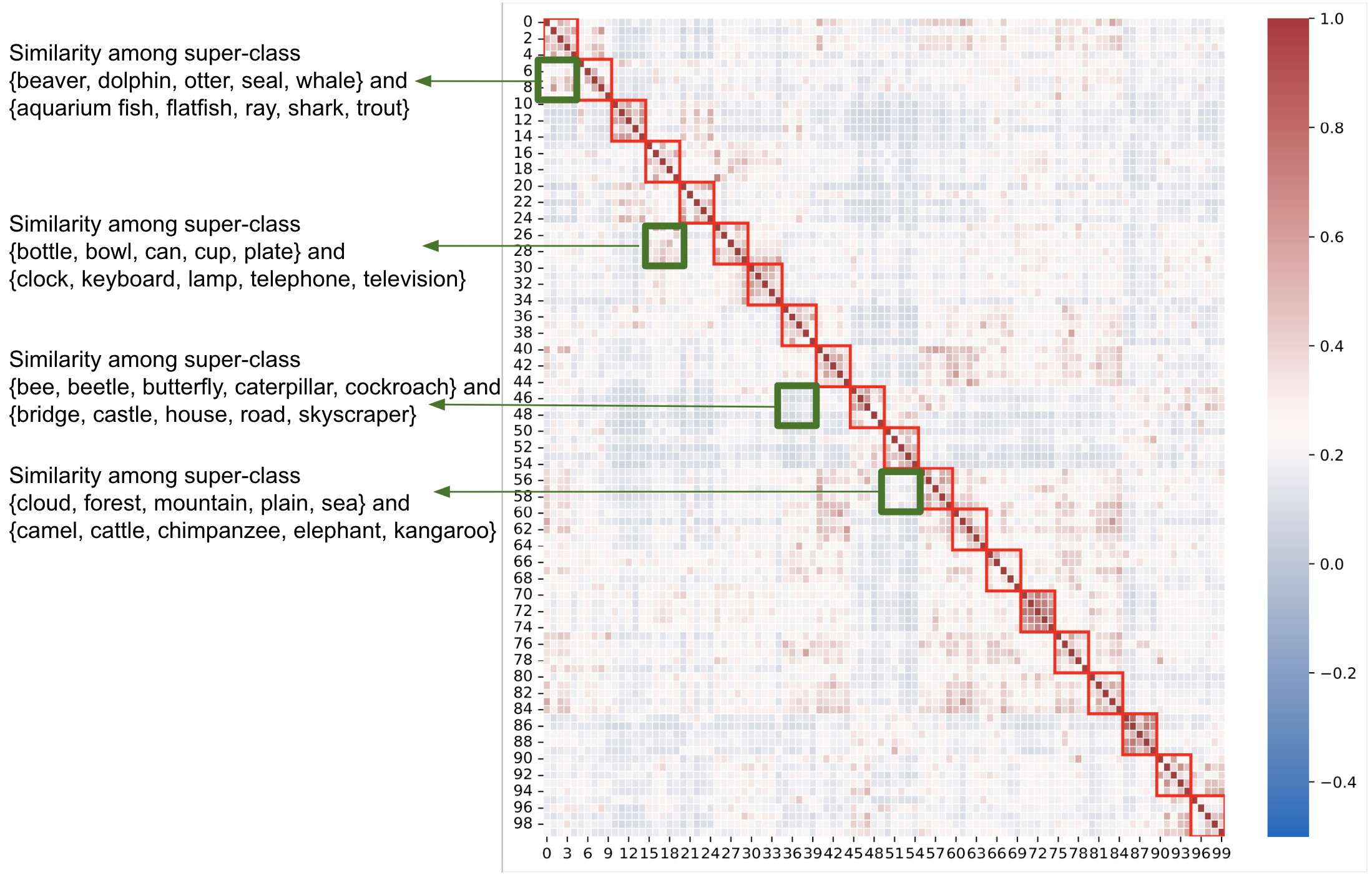}
    \caption{Fine-grained class semantics captured by FedNH.}
    \label{fig:semantics-fednh}
\end{figure}

\clearpage
Lastly, we present a set of prototypes' similarity for different baseline methods in Figure~\ref{fig:cifar100.sim}.
\begin{figure}[!h]
     \centering
     \begin{subfigure}[b]{0.33\textwidth}
         \centering
         \includegraphics[width=6cm]{./figs/fednh_cifar100_similarity.pdf}
         \caption{FedNH}
     \end{subfigure}
     \hfill
     \begin{subfigure}[b]{0.33\textwidth}
         \centering
         \includegraphics[width=6cm]{./figs/fedavg_cifar100_similarity.pdf}
         \caption{FedAvg}
     \end{subfigure}
     \hfill
     \begin{subfigure}[b]{0.33\textwidth}
         \centering
         \includegraphics[width=6cm]{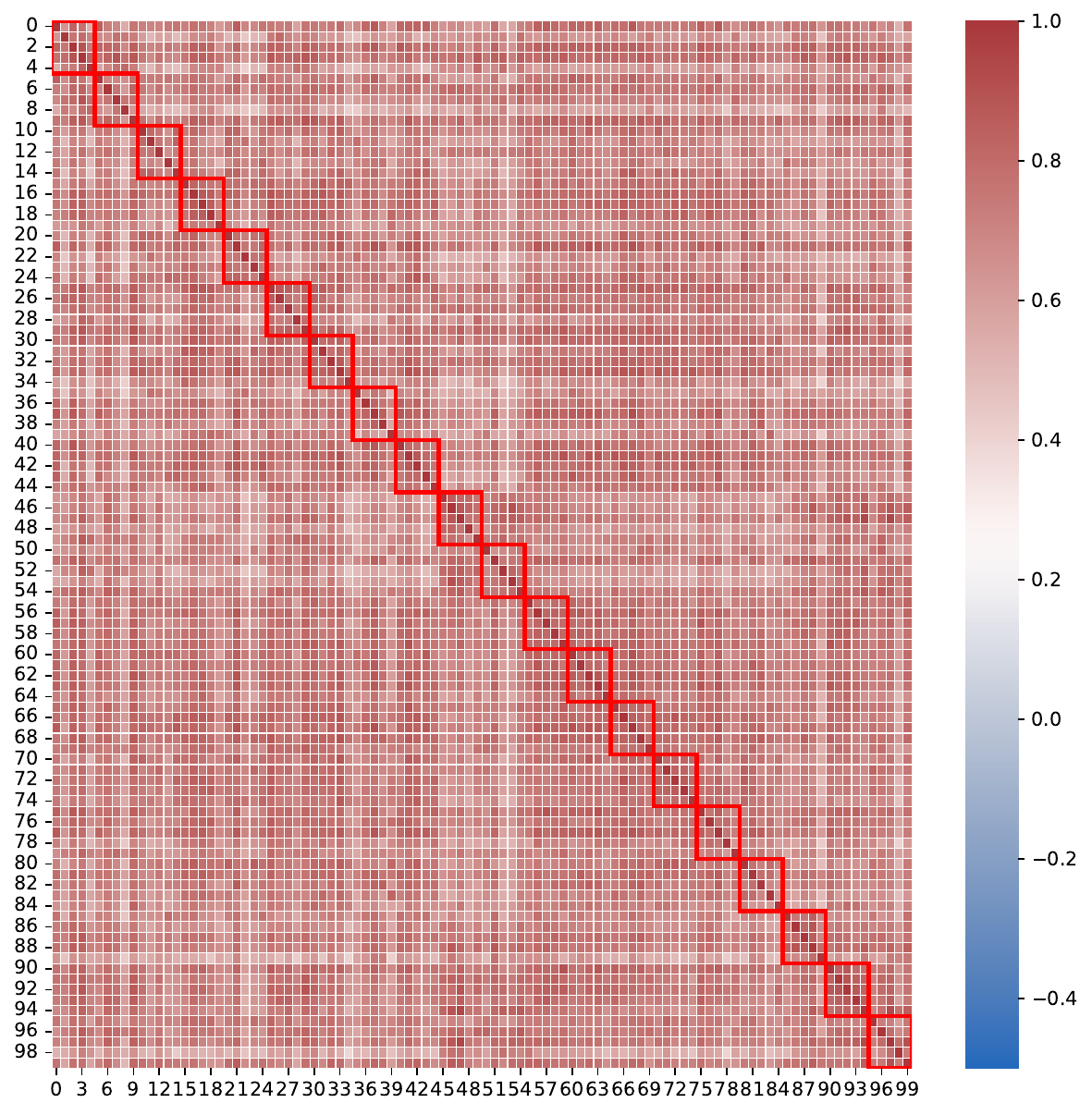}
         \caption{FedProto}
     \end{subfigure}
      \hfill
     \begin{subfigure}[b]{0.33\textwidth}
         \centering
         \includegraphics[width=6cm]{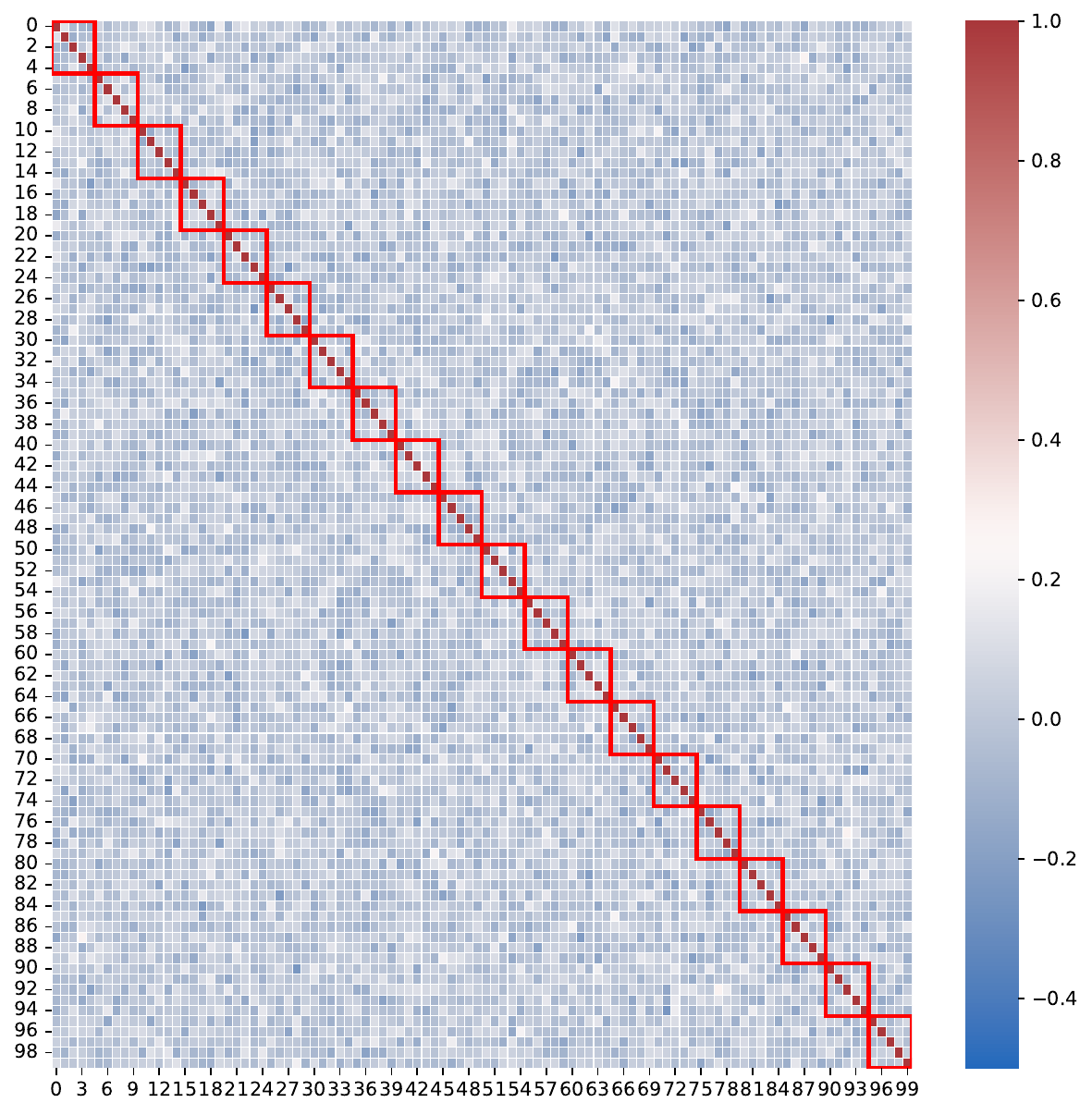}
         \caption{FedBABU}
     \end{subfigure}
           \hfill
     \begin{subfigure}[b]{0.33\textwidth}
         \centering
         \includegraphics[width=6cm]{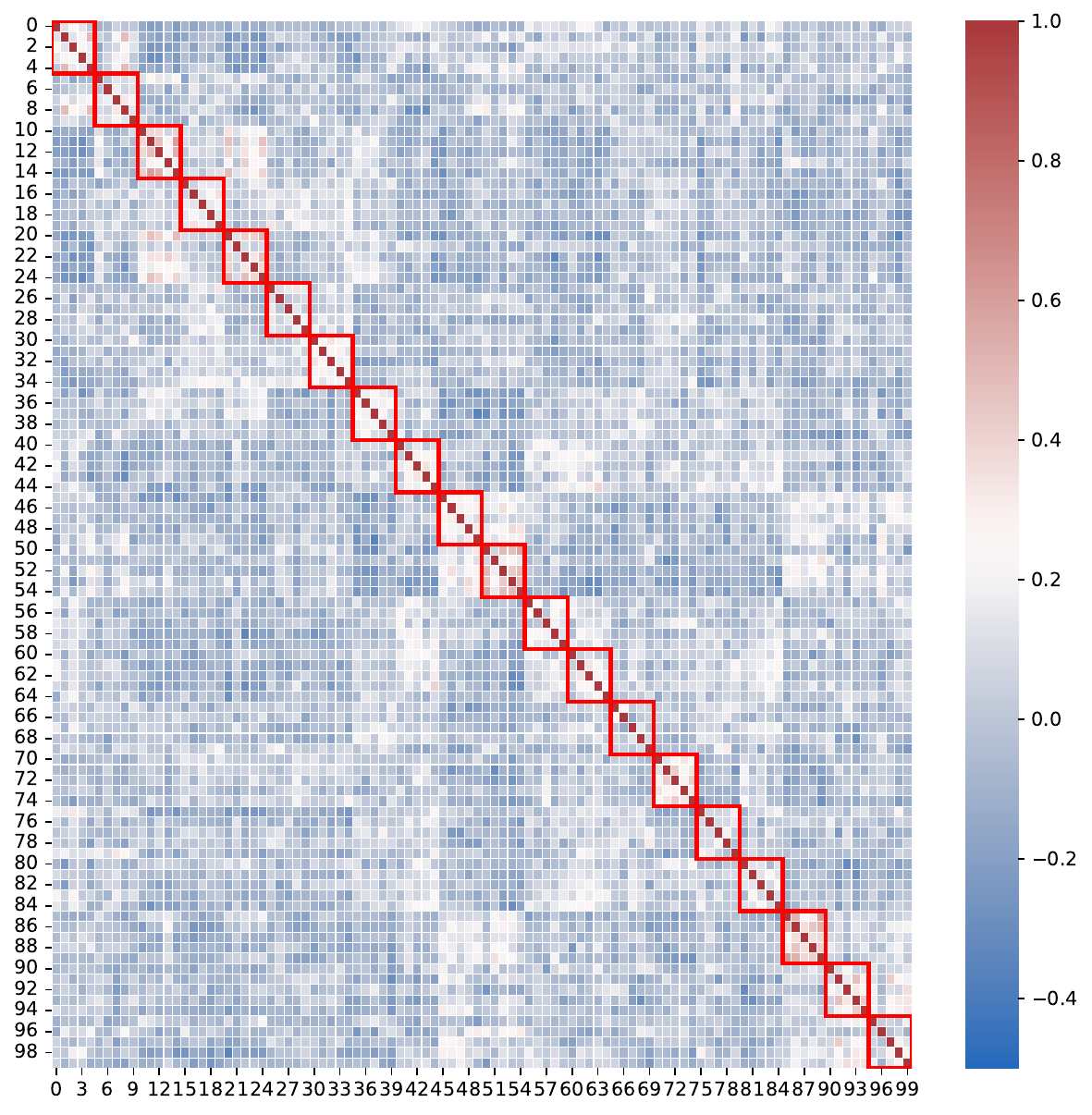}
         \caption{FedROD}
     \end{subfigure}
          \hfill
     \begin{subfigure}[b]{0.33\textwidth}
         \centering
         \includegraphics[width=6cm]{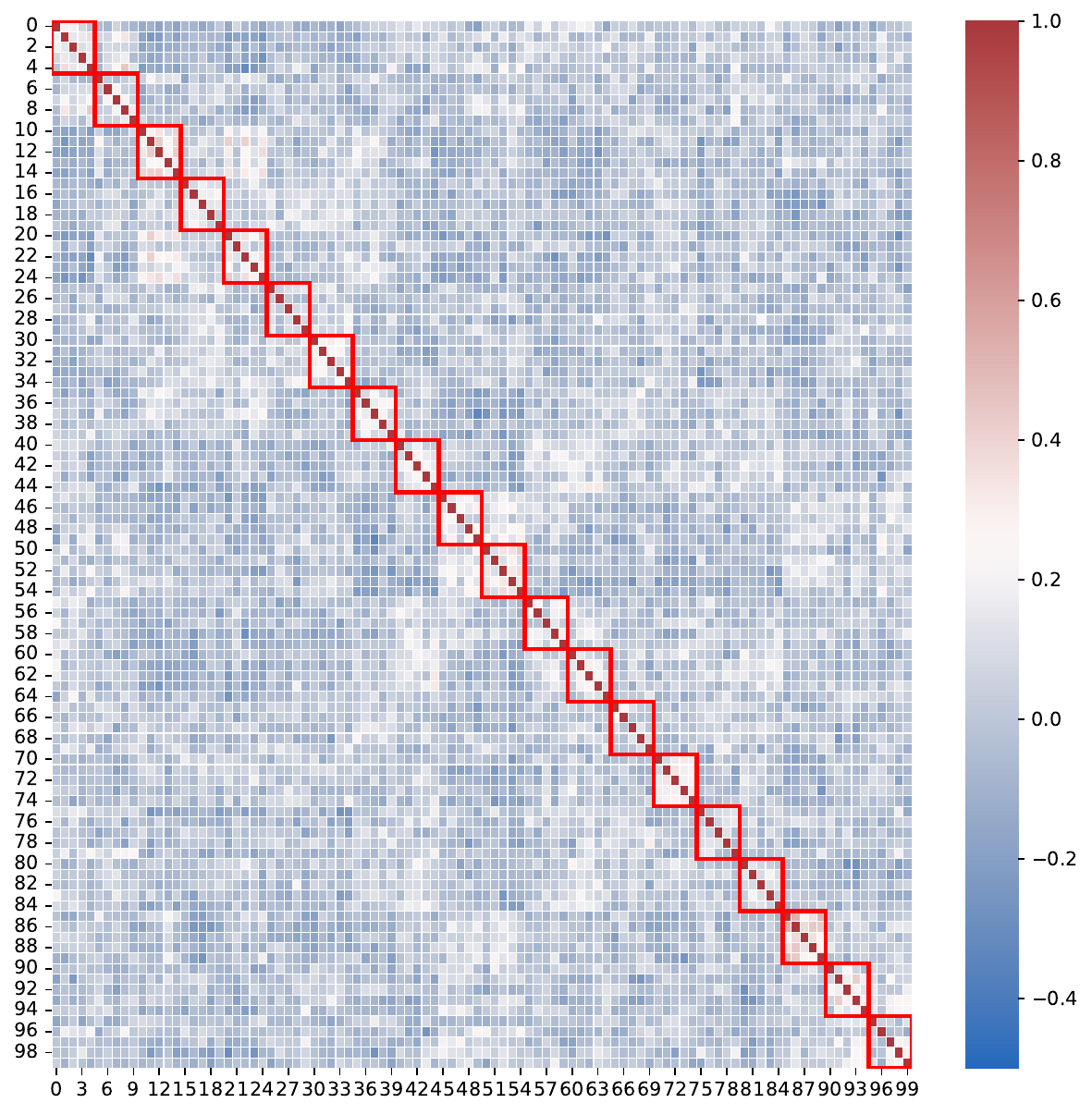}
         \caption{Ditto}
     \end{subfigure}
          \hfill
     \begin{subfigure}[b]{0.33\textwidth}
         \centering
         \includegraphics[width=6cm]{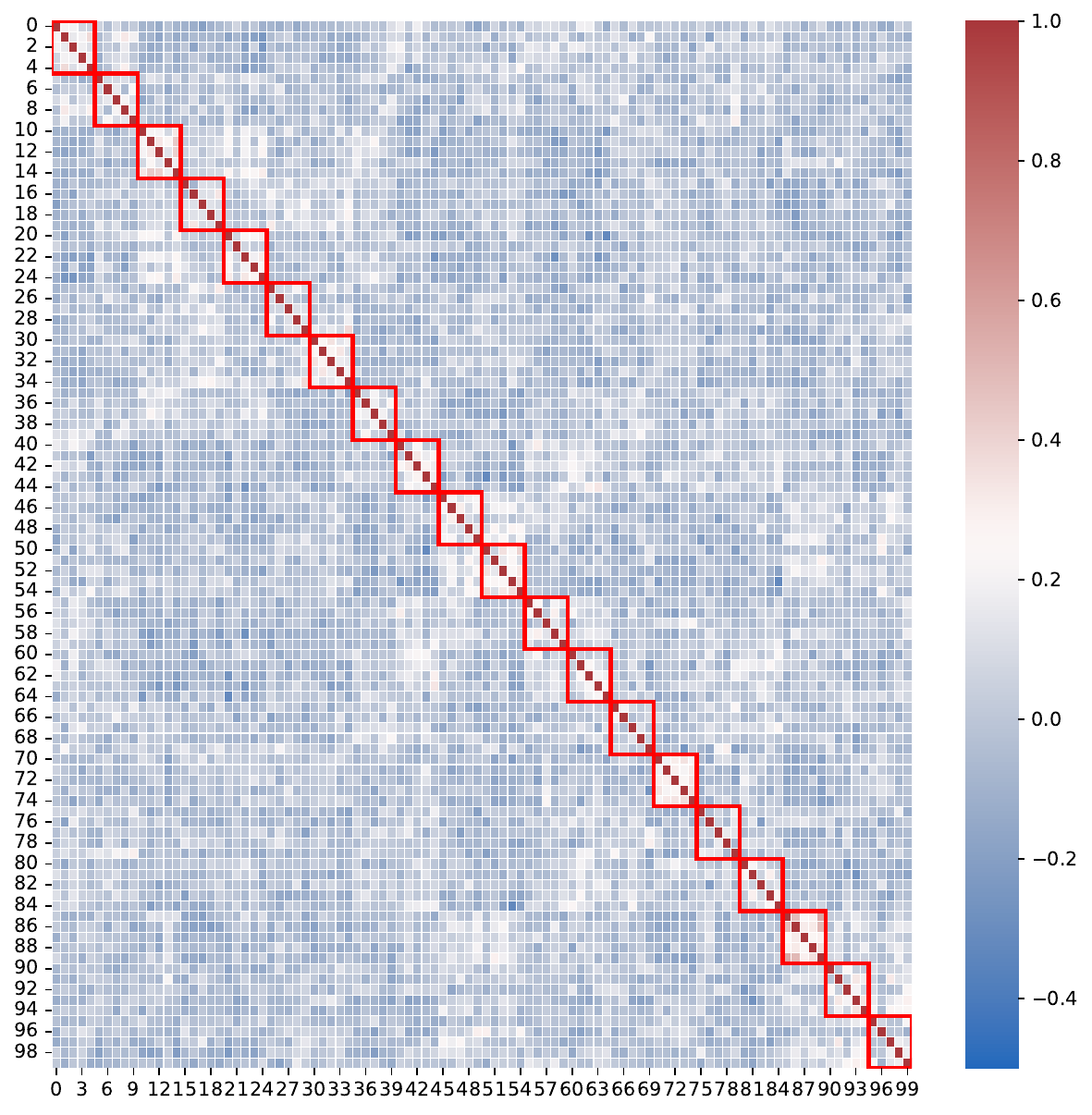}
         \caption{CReFF}
     \end{subfigure}
     \hfill
     \begin{subfigure}[b]{0.33\textwidth}
         \centering
         \includegraphics[width=6cm]{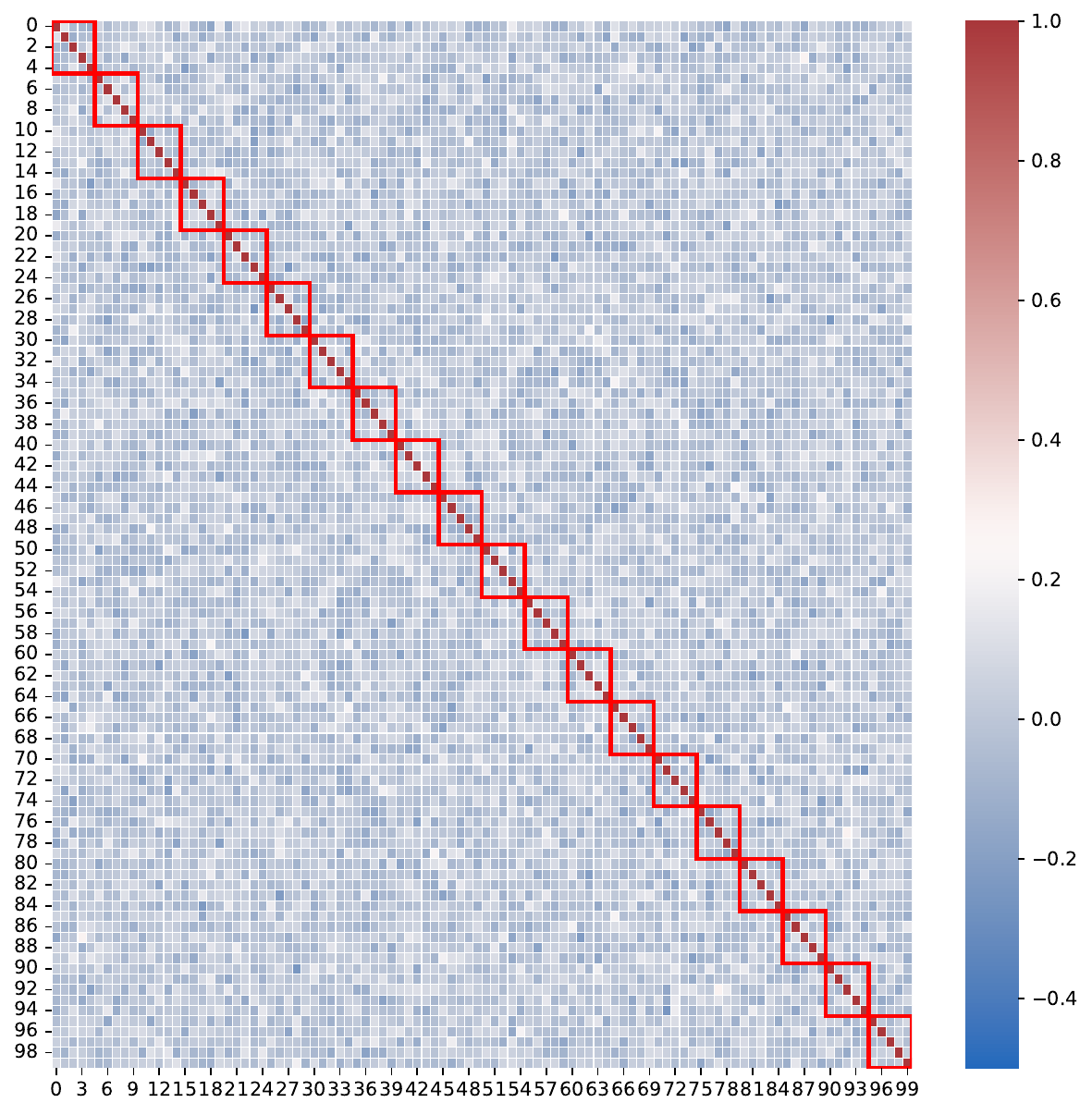}
         \caption{FedPer}
     \end{subfigure}     
     \hfill
     \begin{subfigure}[b]{0.33\textwidth}
         \centering
         \includegraphics[width=6cm]{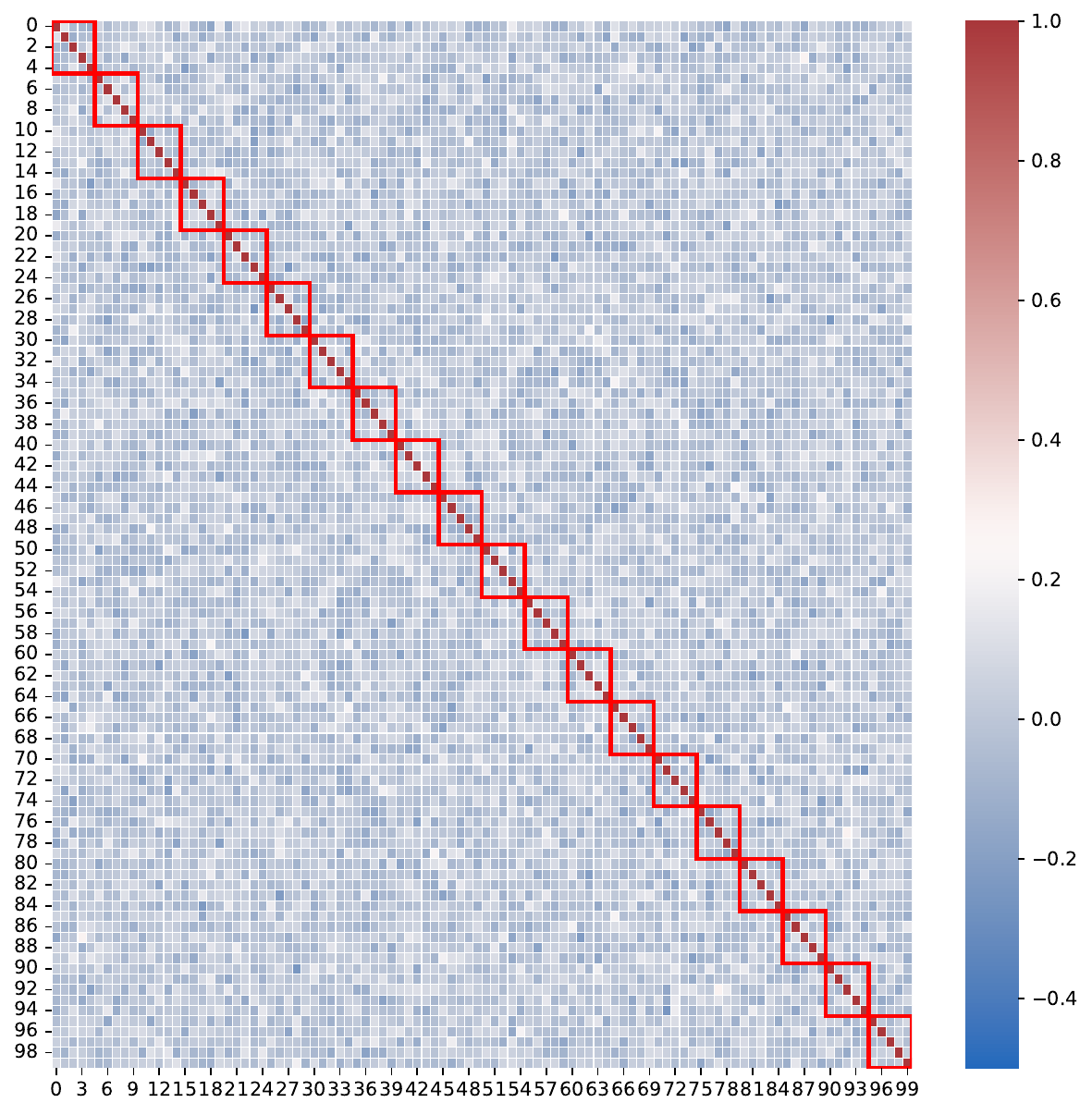}
         \caption{FedRep}
     \end{subfigure}     
    \caption{Class prototypes' pair-wise similarity.}
    \label{fig:cifar100.sim}
\end{figure}

\clearpage
\subsection{Increasing the participate ratio}
We also report the result by increasing the participation ratio of clients from 0.1 to 0.3 in Table~\ref{tab:pratio-0.3}. As one can see, the accuracy in GM, PM(V), and PM(I) will slightly improve.

\begin{table}[!ht]
\begin{tabular}{@{}clcccccc@{}}
\toprule
Dataset                                            & \multicolumn{1}{|c|}{Method}    & \multicolumn{3}{c|}{Dir(0.3)}                                                               & \multicolumn{3}{c}{Dir(1.0)}                                          \\ \midrule
\multicolumn{1}{c|}{}                              & \multicolumn{1}{c|}{}         & GM                    & PM(V)                 & \multicolumn{1}{c|}{PM(L)}                 & GM                    & PM(V)                 & PM(L)                 \\ \midrule
\multicolumn{1}{c|}{\multirow{9}{*}{Cifar10}}      & \multicolumn{1}{l|}{FedAvg}   & 71.16 ± 0.81          & 64.57 ± 2.45          & \multicolumn{1}{c|}{\underline{84.74 ± 1.82}} & 74.91 ± 0.57          & 68.87 ± 1.27          & 79.46 ± 1.03          \\
\multicolumn{1}{c|}{}                              & \multicolumn{1}{l|}{FedPer}   & 57.54 ± 0.29          & 57.20 ± 1.33          & \multicolumn{1}{c|}{80.73 ± 2.31}          & 60.14 ± 0.18          & 56.00 ± 1.23          & 69.65 ± 2.11          \\
\multicolumn{1}{c|}{}                              & \multicolumn{1}{l|}{Ditto}    & 71.16 ± 0.81          & 61.07 ± 2.34          & \multicolumn{1}{c|}{82.88 ± 1.50}          & 74.91 ± 0.57          & 66.23 ± 2.17          & 77.69 ± 1.36          \\
\multicolumn{1}{c|}{}                              & \multicolumn{1}{l|}{FedRep}   & 44.08 ± 0.09          & 55.71 ± 1.11          & \multicolumn{1}{c|}{79.75 ± 2.17}          & 48.11 ± 0.07          & 54.39 ± 1.77          & 68.43 ± 2.16          \\
\multicolumn{1}{c|}{}                              & \multicolumn{1}{l|}{FedProto} & ---                   & 46.08 ± 2.31          & \multicolumn{1}{c|}{71.87 ± 2.45}          & ---                   & 43.19 ± 2.27          & 57.94 ± 1.27          \\
\multicolumn{1}{c|}{}                              & \multicolumn{1}{l|}{CReFF}    & 69.13 ± 0.46          & 64.57 ± 2.45          & \multicolumn{1}{c|}{\underline{84.74 ± 1.82}} & 72.85 ± 0.27          & 68.87 ± 1.27          & 79.46 ± 1.03          \\
\multicolumn{1}{c|}{}                              & \multicolumn{1}{l|}{FedBABU}  & 67.65 ± 0.77          & 61.76 ± 2.16          & \multicolumn{1}{c|}{83.45 ± 1.03}          & 72.84 ± 0.72          & 66.81 ± 1.44          & 78.19 ± 1.80          \\
\multicolumn{1}{c|}{}                              & \multicolumn{1}{l|}{FedROD}   & \textbf{73.41 ± 0.13} & \textbf{66.47 ± 1.27} & \multicolumn{1}{c|}{83.80 ± 1.03}          & \underline{76.10 ± 0.08} & \underline{69.77 ± 1.98} & \underline{78.32 ± 1.76} \\
\multicolumn{1}{c|}{}                              & \multicolumn{1}{l|}{FedNH}    & \underline{72.69 ± 0.43} & \underline{66.35 ± 1.23} & \multicolumn{1}{c|}{\textbf{85.21 ± 2.11}} & \textbf{77.22 ± 0.29} & \textbf{71.04 ± 1.15} & \textbf{80.50 ± 2.14} \\ \midrule
\multicolumn{1}{c|}{\multirow{9}{*}{Cifar 100}}    & \multicolumn{1}{l|}{FedAvg}   & 36.75 ± 0.29          & \underline{32.30 ± 2.45} & \multicolumn{1}{c|}{\underline{50.89 ± 1.82}} & \underline{37.12 ± 0.11} & \underline{28.52 ± 1.27} & \underline{37.83 ± 1.03} \\
\multicolumn{1}{c|}{}                              & \multicolumn{1}{l|}{FedPer}   & 10.29 ± 0.05          & 14.35 ± 1.33          & \multicolumn{1}{c|}{31.04 ± 2.31}          & 10.81 ± 0.02          & 10.23 ± 1.23          & 17.41 ± 2.11          \\
\multicolumn{1}{c|}{}                              & \multicolumn{1}{l|}{Ditto}    & 36.75 ± 0.29          & 29.38 ± 2.34          & \multicolumn{1}{c|}{48.06 ± 1.50}          & 37.12 ± 0.11          & 25.34 ± 2.17          & 34.87 ± 1.36          \\
\multicolumn{1}{c|}{}                              & \multicolumn{1}{l|}{FedRep}   & 4.61 ± 0.01           & 13.22 ± 1.11          & \multicolumn{1}{c|}{29.33 ± 2.17}          & 5.58 ± 0.02           & 9.29 ± 1.77           & 15.98 ± 2.16          \\
\multicolumn{1}{c|}{}                              & \multicolumn{1}{l|}{FedProto} & ---                   & 14.56 ± 2.31          & \multicolumn{1}{c|}{28.25 ± 2.45}          & ---                   & 10.98 ± 2.27          & 16.40 ± 1.27          \\
\multicolumn{1}{c|}{}                              & \multicolumn{1}{l|}{CReFF}    & 17.26 ± 0.12          & \underline{32.30 ± 2.45} & \multicolumn{1}{c|}{\underline{50.89 ± 1.82}} & 18.79 ± 0.13          & \underline{28.52 ± 1.27} & \underline{37.83 ± 1.03} \\
\multicolumn{1}{c|}{}                              & \multicolumn{1}{l|}{FedBABU}  & 34.29 ± 0.25          & 29.05 ± 2.16          & \multicolumn{1}{c|}{47.63 ± 1.03}          & 33.86 ± 0.28          & 25.83 ± 1.44          & 34.73 ± 1.80          \\
\multicolumn{1}{c|}{}                              & \multicolumn{1}{l|}{FedROD}   & 34.17 ± 0.15          & 28.20 ± 1.27          & \multicolumn{1}{c|}{42.82 ± 1.03}          & 35.05 ± 0.10          & 26.92 ± 1.98          & 32.94 ± 1.76          \\
\multicolumn{1}{c|}{}                              & \multicolumn{1}{l|}{FedNH}    & \textbf{42.13 ± 0.17} & \textbf{38.29 ± 1.23} & \multicolumn{1}{c|}{\textbf{55.19 ± 2.11}} & \textbf{43.85 ± 0.14} & \textbf{37.02 ± 1.15} & \textbf{45.45 ± 2.14} \\ \midrule
\multicolumn{1}{c|}{\multirow{9}{*}{TinyImageNet}} & \multicolumn{1}{l|}{FedAvg}   & 35.94 ± 0.27          & 28.24 ± 2.45          & \multicolumn{1}{c|}{45.27 ± 1.82}          & 38.80 ± 0.20          & \underline{29.51 ± 1.27} & \underline{37.57 ± 1.03} \\
\multicolumn{1}{c|}{}                              & \multicolumn{1}{l|}{FedPer}   & 8.26 ± 0.06           & 11.34 ± 1.33          & \multicolumn{1}{c|}{27.15 ± 2.31}          & 8.84 ± 0.06           & 7.73 ± 1.23           & 14.17 ± 2.11          \\
\multicolumn{1}{c|}{}                              & \multicolumn{1}{l|}{Ditto}    & 35.94 ± 0.27          & 25.96 ± 2.34          & \multicolumn{1}{c|}{44.05 ± 1.50}          & 38.80 ± 0.20          & 27.72 ± 2.17          & 36.76 ± 1.36          \\
\multicolumn{1}{c|}{}                              & \multicolumn{1}{l|}{FedRep}   & 2.09 ± 0.02           & 8.80 ± 1.11           & \multicolumn{1}{c|}{22.20 ± 2.17}          & 2.43 ± 0.02           & 5.33 ± 1.77           & 10.24 ± 2.16          \\
\multicolumn{1}{c|}{}                              & \multicolumn{1}{l|}{FedProto} & ---                   & 8.28 ± 2.31           & \multicolumn{1}{c|}{19.18 ± 2.45}          & ---                   & 6.05 ± 2.27           & 10.22 ± 1.27          \\
\multicolumn{1}{c|}{}                              & \multicolumn{1}{l|}{CReFF}    & 28.43 ± 0.47          & 28.24 ± 2.45          & \multicolumn{1}{c|}{45.27 ± 1.82}          & 31.15 ± 0.19          & \underline{29.51 ± 1.27} & \underline{37.57 ± 1.03} \\
\multicolumn{1}{c|}{}                              & \multicolumn{1}{l|}{FedBABU}  & 29.12 ± 0.43          & 22.63 ± 2.16          & \multicolumn{1}{c|}{39.77 ± 1.03}          & 31.43 ± 0.12          & 23.50 ± 1.44          & 31.73 ± 1.80          \\
\multicolumn{1}{c|}{}                              & \multicolumn{1}{l|}{FedROD}   & \underline{37.20 ± 0.25} & \underline{28.03 ± 1.27} & \multicolumn{1}{c|}{\underline{45.11 ± 1.03}} & \underline{38.81 ± 0.19} & 27.43 ± 1.98          & 36.45 ± 1.76          \\
\multicolumn{1}{c|}{}                              & \multicolumn{1}{l|}{FedNH}    & \textbf{38.15 ± 0.17} & \textbf{31.57 ± 1.23} & \multicolumn{1}{c|}{\textbf{46.53 ± 2.11}} & \textbf{39.02 ± 0.15} & \textbf{30.76 ± 1.15} & \textbf{37.97 ± 2.14} \\ \bottomrule
\end{tabular}
\caption{Comparison of testing accuracy with participation ratio being 0.3.
The best results are in bold font, while the second best results are underlined. The lines ``---'' represent results are not available. The numbers (mean ± std) are the average of three independent runs.}
\label{tab:pratio-0.3}
\end{table}

\subsection{Discussion on computation costs}
It can be seen that FedNH adds minimum computation overhead compared with the state-of-the-art methods, for example, FedROD\cite{chen2021bridging}, which needs to train two classification heads (or even a sub-network to generate one of the classification heads). Instead, FedNH method only requires additional computation on prototypes' updates \eqref{eq:local-proto-update} and \eqref{eq:prototype-update}.

\clearpage
\section{Convergence Analysis}\label{appendix:complated-convergence-analysis}
In this section, we present the complete convergence analysis for Algorithm~\ref{alg:FedNH} under the full participation assumption. 

\begin{lemma}\label{lemma:bounded-full-grad}
For any $k$ and $W$,  $\norm{\grad_\theta F_k(\theta;W)} \leq M_G$
\end{lemma}
\begin{proof}
It follows from the Assumption~\ref{ass:first-second-momentum} and \ref{ass:bounded-gradient} and the Jensen's inequality that
$$
  \norm{\grad_\theta F_k(\theta;W)} = \norm{\Embb[G_k(\theta;W)]}\leq \sqrt{\Embb\norm{G_k(\theta;W)}^2} \leq M_G,  
$$
\end{proof}

\begin{lemma}\label{lemma:bound-grad}
For $W\in\R{|\Ccal|\times d_2}$. Define function 
$$
h(W):=\frac{1}{n} \sum_{i=1}^{n}-\log\frac{\exp(W_{y_i}f(\theta;x_i,y_i))}{\sum_{j=1}^{|\Ccal|}\exp(W_jf(\theta;x_i,y_i))},
$$
where $y_i\in[c]$ for all $i$, $W_{y_i}\in\R{d_2}$ is a row vector, and $f(\theta;x_i,y_i)\in\R{d_2}$ is a column vector. Then
$\norm{\grad h(W)}\leq |C|$.
\end{lemma}
\begin{proof}
Denote $A=\sum_{j=1}^{|\Ccal|}\exp(W_jf(\theta;x_i,y_i)^)\in\R{+}$ and $\phi_i=f(\theta;x_i,y_i)\in\R{d_2}$. By the construction of the algorithm, i.e., the latent representation $\phi_i$ is normalized, then $\norm{\phi_i}=1$. For any $c\in[C]$,
$$
\begin{aligned}
\frac{\partial h(W)}{\partial W_c} 
&= -\frac{1}{n}\sum_{\{i:y_i=c\}}\frac{A}{\exp(W_{c}\phi_i)}\left(\frac{\exp(W_{c}\phi_i)A-\exp(W_{c}\phi_i)^2}{A^2}\phi_i\right)  + \frac{1}{n}\sum_{\{i:y_i\neq c\}} \frac{A}{\exp(W_{y_i}\phi_i)}\frac{\exp(W_{c}\phi_i)\exp(W_{y_i}\phi_i)}{A^2}\phi_i\\
& = -\frac{1}{n}\left(\sum_{\{i:y_i=c\}}\left(1-\frac{\exp(W_{c}\phi_i)}{A}\right)\phi_i - \sum_{\{i:y_i\neq c\}}\frac{\exp(W_{c}\phi_i)}{A}\phi_i\right)
\end{aligned}
$$
which implies
$$
\norm{\frac{\partial h(W)}{\partial W_c}}\leq \frac{1}{n}\left(\sum_{i=1}^n\norm{\phi_i}\right)= 1.
$$
Therefore, $\norm{\grad h(W)}\leq |C|$.
\end{proof}

The following lemma establishes one communication round progress. The proof techniques are borrowed from \cite{tan2022fedproto, yu2019parallel}.
\begin{lemma}
Let the algorithmic choices made to Algorithm~\ref{alg:FedNH} are $\rho\in(0,1)$, $|\Scal^t|=K$, and  $\alpha_k^t=\frac{1}{K}$ for any $t\in[R]$ and $k\in[K]$. Then
\begin{equation}\label{eq:one-round-progress}
\Embb[F_k(\theta_{k}^{t,J};W^t) - F_k(\theta_{k}^{t-1,J};W^{t-1})]  \leq  -\left(\eta_t-\frac{L_g\eta_t^2}{2}\right)\sum_{j=1}^{J}\Embb\left[\norm{\grad_{\theta} F_k(\theta_{k}^{t,j-1};W^t)}^2\right] + \frac{L_gJ\eta_t^2\sigma^2}{2} + (1-\rho)\kappa\eta_t,
\end{equation}
where $\kappa = 2JM_GM_\iota + 3JM_GM_pM_f|\Ccal|$, and $M_p$ and $M_\iota$ are positive constants.
\end{lemma}
\begin{proof}
First, we establish the standard sufficient decrease result, where the amount of progress can be made (in expectation) by performing $E$ steps of local stochastic gradient descent.
It follows from Assumption~\ref{ass:smoothness} and equation~\eqref{eq:sgd-update},
\begin{align}
F_k(\theta_{k}^{t,j};W^t)  
&\leq  F_k(\theta_{k}^{t,j-1};W^t) + \grad_{\theta} F_k(\theta_{k}^{t,j-1};W^t)(\theta_{k}^{t,j}-\theta_{k}^{t,j-1}) + \frac{L_g}{2}\norm{\theta_{k}^{t,j}-\theta_{k}^{t,j-1}}^2 \nonumber\\
&=  F_k(\theta_{k}^{t,j-1};W^t) - \eta_t\grad_{\theta} F_k(\theta_{k}^{t,j-1};W^t)G_k(\theta_{k}^{t,j-1};W^t) + \frac{L_g\eta_t^2}{2}\norm{G_k(\theta_{k}^{t,j-1};W^t)}^2 \label{eq:suff1}
\end{align}   
Taking expectation with respect to randomness from sub-sampling $\Dcal_k$ on both sides of \eqref{eq:suff1} and together with Assumption~\ref{ass:first-second-momentum}, 
\begin{align}
\Embb[F_k(\theta_{k}^{t,j};W^t)] 
&\leq F_k(\theta_{k}^{t,j-1};W^t) - \eta_t\norm{\grad_{\theta} F_k(\theta_{k}^{t,j-1};W^t)}^2 + \frac{L_g\eta_t^2}{2}\Embb\norm{G_k(\theta_{k}^{t,j-1};W^t)}^2 \nonumber\\
& \leq F_k(\theta_{k}^{t,j-1};W^t) - \eta_t\norm{\grad_{\theta} F_k(\theta_{k}^{t,j-1};W^t)}^2 + \frac{L_g\eta_t^2}{2}\left[\norm{\grad_{\theta} F_k(\theta_{k}^{t,j-1};W^t)}^2 + \sigma^2\right] \nonumber\\
& =  F_k(\theta_{k}^{t,j-1};W^t) - \left(\eta_t-\frac{L_g\eta_t^2}{2}\right)\norm{\grad_{\theta} F_k(\theta_{k}^{t,j-1};W^t)}^2 + \frac{L_g\eta_t^2\sigma^2}{2}\label{eq:suff2},
\end{align}    
where the second last using the fact that $\Embb\norm{X}^2= \Embb\norm{X-\Embb X}^2 + \norm{\Embb[X]}^2$. Finally, for \eqref{eq:suff2}, summing over the index $j\in[J]$, one reaches to
\begin{equation}\label{eq:suff-done}
    \begin{aligned}
    \Embb[F_k(\theta_{k}^{t,J};W^t)]  \leq F_k(\theta_{k}^{t,0};W^t) - \left(\eta_t-\frac{L_g\eta_t^2}{2}\right)\sum_{j=1}^{J}\norm{\grad_{\theta} F_k(\theta_{k}^{t,j-1};W^t)}^2 + \frac{L_gJ\eta_t^2\sigma^2}{2}.
    \end{aligned}
\end{equation}
Next, we bound $F_k(\theta_{k}^{t,0};W^t) - F_k(\theta_{k}^{t-1,E};W^{t-1})$, which characterizes the difference between the stage that the client just finishes $J$ local updates and the stage that the client recives new prototype and the representation neural network parameter from the server in function value.
\begin{align}
& F_k(\theta_{k}^{t,0};W^t) - F_k(\theta_{k}^{t-1,J};W^{t-1})\nonumber\\
& = F_k(\theta_{k}^{t,0};W^t) - F_k(\theta_{k}^{t-1,J};W^{t}) + F_k(\theta_{k}^{t-1,J};W^{t}) - F_k(\theta_{k}^{t-1,J};W^{t-1})\nonumber\\
& \leq |F_k(\theta_{k}^{t,0};W^t) - F_k(\theta_{k}^{t-1,J};W^{t})| + |F_k(\theta_{k}^{t-1,J};W^{t}) - F_k(\theta_{k}^{t-1,J};W^{t-1})|\label{eq:diff}
\end{align}

\textbf{We first bound $|F_k(\theta_{k}^{t,0};W^t) - F_k(\theta_{k}^{t-1,J};W^{t})|$.} It follows from Lemma~\ref{lemma:bounded-full-grad} that  $F_k(\theta;W)$ is $M_G$-Lipschtiz continuous with respect to the first argument. Therefore,
\begin{equation}\label{eq:part1-1}
|F_k(\theta_{k}^{t,0};W^t) - F_k(\theta_{k}^{t-1,J};W^{t})|\leq M_G\norm{\theta_{k}^{t,0}-\theta_{k}^{t-1,J}}.
\end{equation}
Since $\theta_{k}^{t,0}=\frac{1}{K}\sum_{k=1}^K\theta_{k}^{t-1,J}$, using the same technique from the Lemma 1 in \cite{yu2019parallel}, one has 
\begin{equation}\label{eq:bound-diff-local-gloabl-model}
\Embb\left[\norm{\theta_{k}^{t,0} - \theta_{k}^{t-1,J}}\right]\leq 2\eta_tJM_G.
\end{equation}
Combining \eqref{eq:part1-1} and \eqref{eq:bound-diff-local-gloabl-model} leads to
\begin{equation}\label{eq:part1-done}
    \Embb[|F_k(\theta_{k}^{t,0};W^t) - F_k(\theta_{k}^{t-1,J};W^{t})|]\leq 2\eta_tJM_G^2.
\end{equation}
\textbf{We then bound $|F_k(\theta_{k}^{t-1,J};W^{t}) - F_k(\theta_{k}^{t-1,J};W^{t-1})|$.} It follows from Lemma~\ref{lemma:bound-grad}, the $\grad_W F_k(\theta;W)$ is bounded, which further implies $F_k$ is $M_p$-Lipschtiz continuous with respect to the second argument $W$ (We are free to choose $M_p$ such that $M_p>M_G$.). It follows from Lipschtiz continuity of $F_k$ with respect to $W$ and the definition of $W^t$
\begin{align}
& |F_k(\theta_{k}^{t-1,J};W^{t}) - F_k(\theta_{k}^{t-1,J};W^{t-1})|\nonumber\\
& \leq M_p\norm{W^{t} - W^{t-1}}\nonumber\\
& \leq \frac{ M_p}{K}\sum_{k=1}^{K}\sum_{c=1}^{|\Ccal|}\left((1-\rho)\norm{\mu^{t}_{k,c} - \mu^{t-1}_{k,c}} + \rho\norm{W_c^{t-1}-W_c^{t-2}}\right)\nonumber\\
& \overset{(a)}{=} \frac{ M_p}{K}\sum_{k=1}^{K}\sum_{c=1}^{|\Ccal|}\sum_{l=0}^{t-1}\rho^{t-l}(1-\rho)\norm{\mu^{t-l}_{k,c} - \mu^{t-l-1}_{k,c}}\nonumber\\
& \leq M_p|\Ccal|(1-\rho)\sum_{l=0}^{t-1}\norm{\mu^{t-l}_{k,c} - \mu^{t-l-1}_{k,c}}\label{eq:part2-1}
\end{align}
where $(a)$ holds by the recursion over the definition of $W^t$ and we define $\mu_{k,c}^{0}=0$ for all $k\in[K]$ and $c\in[|\Ccal|]$.

By the definition of $\mu_{k,c}^t$ and Assumption~\ref{ass:bounded-gradient}, one has
\begin{align}
\norm{\mu^{t}_{k,c} - \mu^{t-1}_{k,c}}
& = \norm{ \frac{1}{|\{i: y_i=c, (x_i,y_i)\in\Dcal_k\}|}\sum_{\{i: y_i=c, (x_i,y_i)\in\Dcal_k\}}\left(r_{k,i}^{t}-r_{k,i}^{t-1}\right)}\nonumber\\
&\leq \frac{1}{|\{i: y_i=c, (x_i,y_i)\in\Dcal_k\}|}\sum_{\{i: y_i=c, (x_i,y_i)\in\Dcal_k\}}\norm{
f(\theta_{k}^{t,J};x_i,y_i)-f(\theta_{k}^{t-1,J};x_i,y_i)}\nonumber\\
&\leq \frac{1}{|\{i: y_i=c, (x_i,y_i)\in\Dcal_k\}|}\sum_{\{i: y_i=c, (x_i,y_i)\in\Dcal_k\}}M_f\norm{\theta_{k}^{t,J}-\theta_{k}^{t-1,J}}\label{eq:part2-2}
\end{align}

It can be seen that
$$
    \norm{\theta_{k}^{t,J}-\theta_{k}^{t-1,J}}\leq \sum_{j=1}^{J}\norm{\theta_{k}^{t,j}-\theta_{k}^{t,j-1}} + \norm{\theta_{k}^{t,0}-\theta_{k}^{t-1,J}} = \eta_t\sum_{j=1}^{J}\norm{G_k(\theta^{j-1};W^t)} + \norm{\theta_{k}^{t,0}-\theta_{k}^{t-1,J}}.
$$
Taking expectation on both sides and together with \eqref{eq:bound-diff-local-gloabl-model}, one reaches to
\begin{equation}\label{eq:part2-3}
    \Embb\left[\norm{\theta_{k}^{t,J}-\theta_{k}^{t-1,J}}\right]\leq 3\eta_tJM_G.
\end{equation}

Combine \eqref{eq:part2-2} and \eqref{eq:part2-3} to get
\begin{equation}\label{eq:part2-4}
   \Embb\left[\norm{\mu^{t}_{k,c} - \mu^{t-1}_{k,c}}\right]\leq 3n_kJM_fM_G\eta_t
\end{equation}

From \eqref{eq:part2-1} and \eqref{eq:part2-4}
\begin{equation}\label{eq:part2-done}
   \Embb[|F_k(\theta_{k}^{t-1,J};W^{t}) - F_k(\theta_{k}^{t-1,J};W^{t-1})|] \leq  3(1-\rho)\eta_tJM_GM_pM_f|\Ccal|.
\end{equation}

Combine \eqref{eq:diff}, \eqref{eq:part1-done}, and \eqref{eq:part2-done}
one gets
\begin{equation}\label{eq:diff-done}
\Embb[F_k(\theta_{k}^{t,0};W^t) - F_k(\theta_{k}^{t-1,J};W^{t-1})]\leq 3(1-\rho)\eta_tJM_GM_pM_f|\Ccal| + 2(1-\rho)\eta_tJM_GM_{\iota}  = (1-\rho)\kappa \eta_t,
\end{equation}
where $M_{\iota}=M_G/(1-\rho)$.
Finally, combine \eqref{eq:suff-done} with \eqref{eq:diff-done}, we reach to the desired result
$$
\Embb[F_k(\theta_{k}^{t,J};W^t) - F_k(\theta_{k}^{t-1,J};W^{t-1})]  \leq  -\left(\eta_t-\frac{L_g\eta_t^2}{2}\right)\sum_{j=1}^{J}\Embb\left[\norm{\grad_{\theta} F_k(\theta_{k}^{t,j-1};W^t)}^2\right] + \frac{L_gJ\eta_t^2\sigma^2}{2} + (1-\rho)\kappa\eta_t.
$$
\end{proof}
\begin{lemma}\label{lemma:bounded-below}
When $\rho \in (1-\frac{M_G^2J}{\kappa}, 1)$ and set $\eta_t \leq \frac{2(M_G^2J-(1-\rho)\kappa)}{L_g(M_G^2 + J\sigma^2)}$ for all $t\in[R]$, then there exits a constant $F_k^*$ such that $F_k^*\leq \inf_t\{\Embb[F_k(\theta_k^{t,J})]\}$.
\end{lemma}
\begin{proof}
First, by the definition $\kappa$, one can verify that $\frac{M_G^2J}{\kappa}\in(0,1)$ and $\eta_t>0$. Then substitute $\eta_t$ in \eqref{eq:one-round-progress}, one can see $\Embb[F_k(\theta_{k}^{t,J};W^t) - F_k(\theta_{k}^{t-1,J};W^{t-1})] < 0$. Together the Assumption~\ref{ass:smoothness} that $F_k$ is bounded below, one can see $\inf_t\{\Embb[F_k(\theta_k^{t,J})]\}$ must be bounded below. Hence, the existence of $F_k^*$ is guaranteed.
\end{proof}


\begin{theorem}\label{thm:formal}
Let the $k$th client uniformly at random returns an element from $\{\theta_{k}^{t,j}\}$ as the solution denoted as $\theta_k^*$. Further, let $W^*$ share the same round index as $\theta_k^*$. For any $\epsilon>0$, set $\rho\in \left(\max\left(1-\frac{M_G^2J}{\kappa}, 1-\frac{\epsilon J}{\kappa}\right), 1\right)$ and $\eta_t\equiv \eta \in \left(0, \min\left( \frac{2\epsilon-2(1-\rho\kappa/J)}{\epsilon L_g+\sigma^2L_g, }, \frac{2(M_G^2J-(1-\rho)\kappa)}{L_g(M_G^2 + J\sigma^2)}\right)\right)$, where $\kappa>0$ is a constant.
If $R>\frac{2(F_k(\theta_k^{0,0};W^0) - F_k^*)}{J\eta[\epsilon(2-L_g\eta) - (L_g^2\sigma^2\eta + 2(1-\rho)\kappa/J]}$, one has
$$
\Embb\left[\norm{\grad_{\theta} F_k(\theta_{k}^{*};W^*)}^2\right] \leq \epsilon.
$$
\end{theorem}

\begin{proof}
Set $\eta_t\equiv \eta$ in \eqref{eq:one-round-progress} and rearrange it, one gets
\begin{equation}\label{eq:step1}
\left(\eta-\frac{L_g\eta^2}{2}\right)\sum_{j=1}^{J}\Embb\left[\norm{\grad_{\theta} F_k(\theta_{k}^{t,j-1};W^t)}^2\right]
  \leq \Embb[F_k(\theta_{k}^{t-1,J};W^{t-1}) - F_k(\theta_{k}^{t,J};W^t) ]  + \frac{L_gJ\eta^2\sigma^2}{2} + (1-\rho)\kappa\eta,
\end{equation}
Summing over $t$ and dividing $RJ(\eta-\frac{L_g\eta^2}{2})$ on both sides of \eqref{eq:step1}, one gets
$$
\begin{aligned}
\frac{1}{RJ}\sum_{t=1}^{R}\sum_{j=1}^{J}\Embb\left[\norm{\grad_{\theta} F_k(\theta_{k}^{t,j-1};W^t)}^2\right]
& \leq \frac{\frac{1}{RJ}\sum_{t=1}^{R}\Embb[F_k(\theta_{k}^{t-1,J};W^{t-1}) - F_k(\theta_{k}^{t,J};W^t) ]+ \frac{L_g\eta^2\sigma^2}{2} + \frac{(1-\rho)\kappa\eta}{J} }{\eta-\frac{L_g\eta^2}{2}}\\
& \leq \frac{\frac{2 (F_k(\theta_k^{0,0};W^0) - F_k^*)}{RJ} + L_g\eta^2\sigma^2 + \frac{2(1-\rho)\kappa\eta}{J}}{2\eta-L_g\eta^2},
\end{aligned}
$$
where the last inequality follows from the definition of $F_k^*$. Then taking expectation with respect to the $R$ and $J$
$$
\Embb\left[\norm{\grad_{\theta} F_k(\theta_{k}^{*};W^*)}^2\right] \leq \frac{\frac{2 (F_k(\theta_k^{0,0};W^0) - F_k^*)}{RJ} + L_g\eta^2\sigma^2 + \frac{2(1-\rho)\kappa\eta}{J}}{2\eta-L_g\eta^2}.
$$
Substitute $R$, then one reaches the result.
\end{proof}

\end{document}